%% file: main.tex
\documentclass[letterpaper]{article} 
\usepackage{aaai24}  
\usepackage{times}  
\usepackage{helvet}  
\usepackage{courier}  
\usepackage[hyphens]{url}  
\usepackage{graphicx} 
\usepackage{natbib}  
\usepackage{caption} 
\usepackage{subcaption}
\usepackage{amsmath}
\usepackage{amssymb}
\usepackage{mathtools}
\usepackage{amsthm}

\usepackage[utf8]{inputenc} 
\usepackage{booktabs}       
\usepackage{amsfonts}       
\usepackage{nicefrac}       
\usepackage{microtype}      
\usepackage{xcolor}         

\usepackage{algorithm}
\usepackage{algorithmic}

\DeclareMathOperator*{\argmax}{arg\,max}

\DeclareMathOperator*{\clip}{clip}

\newtheorem{lemma}{Lemma}
\newtheorem{theorem}{Theorem}

\title{Personalized Reinforcement Learning with a Budget of Policies}

\author{
    Dmitry Ivanov,
    Omer Ben-Porat
}

\affiliations{
    Technion, Israel\\
    divanov@campus.technion.ac.il, omerbp@technion.ac.il
}

\begin{document}

\maketitle

\begin{abstract}
Personalization in machine learning (ML) tailors models' decisions to the individual characteristics of users. While this approach has seen success in areas like recommender systems, its expansion into high-stakes fields such as healthcare and autonomous driving is hindered by the extensive regulatory approval processes involved. To address this challenge, we propose a novel framework termed represented Markov Decision Processes (r-MDPs) that is designed to balance the need for personalization with the regulatory constraints. In an r-MDP, we cater to a diverse user population, each with unique preferences, through interaction with a small set of representative policies. Our objective is twofold: efficiently match each user to an appropriate representative policy and simultaneously optimize these policies to maximize overall social welfare. We develop two deep reinforcement learning algorithms that efficiently solve r-MDPs. These algorithms draw inspiration from the principles of classic K-means clustering and are underpinned by robust theoretical foundations. Our empirical investigations, conducted across a variety of simulated environments, showcase the algorithms' ability to facilitate meaningful personalization even under constrained policy budgets. Furthermore, they demonstrate scalability, efficiently adapting to larger policy budgets.
\end{abstract}

\section{Introduction}
\label{sec:introduction}

Personalization in machine learning (ML) tailors the decision-making process of a model to align with an individual's unique characteristics and preferences. This approach is typically realized either through individual-specific models or by fine-tuning a universal model with personal data. It is successfully applied in various fields such as recommender systems~\cite{shepitsen2008personalized,lee2012prea,yao2020rlper}, natural language processing \cite{wu2023personalized}, healthcare~\cite{ayer2012or}, and financial services \cite{capponi2022personalized}. For instance, in recommender systems, personalization enables models to suggest products or services to users based on their individual purchase histories and browsing behaviors.

Despite these successes, the integration of personalization in ML into critical sectors like healthcare and autonomous driving, where errors can lead to severe consequences, remains limited. Products driven by ML must undergo extensive regulatory review and approval processes to ensure they offer benefits that significantly outweigh potential risks for their intended user populations. The review process, as exemplified by the Artificial Pancreas that monitors and controls glucose levels \cite{breton2020randomized}, involves a thorough evaluation by regulatory bodies like the Food and Drug Administration (FDA) to affirm the balance of benefits and risks. The FDA's prolonged authorization of a comprehensive Artificial Pancreas solution, spanning several years \cite{jdrf_blog_2022}, underscores the complexity and rigor of such evaluations. Similarly, autonomous vehicles employing reinforcement learning (RL) systems for navigation confront formidable challenges. Despite accumulating millions of hours in test driving, these vehicles must pass meticulous review and audit processes before entering production. The integration of personalization in these systems, necessitating the assessment of individualized policies for safety and efficacy, further complicates the regulatory landscape.

The challenges observed in the aforementioned domains reflect a broader issue: in high-risk and complex environments, the primary obstacle often lies not in data acquisition or hardware limitations, but in the protracted regulatory approval process. This bottleneck necessitates an innovative approach that balances regulatory feasibility with the benefits of personalization. Our proposed solution is to develop a limited number of tailored policies, each catering to a specific user group, thereby streamlining the review process while maintaining the personalization advantage.

In this context, we model our scenario as a Markov decision process (MDP) involving a population of $n$ users, or agents, each characterized by a unique reward function reflecting their preferences within the MDP. Ideally, each agent would be offered a distinct personalized policy. However, given the regulatory constraints highlighted above, we propose a more practical strategy: the development of at most $k<n$ policies. Under this framework, each agent selects the most appropriate policy from this smaller set.

To formalize this concept, we introduce a novel abstraction: the represented MDP (r-MDP). In r-MDP, agents do not directly engage with the MDP. Instead, they are aligned with $k$ representatives, each managing a single policy within the MDP. Agents associated with the same representative adhere to the same policy. The goal is to optimally match agents to representatives and train these representative policies to maximize the overall social welfare of the agents. This approach addresses the regulatory challenges by reducing the number of policies requiring approval, thereby facilitating a more efficient review process without significantly compromising the personalization benefits.

Our proposed pipeline can be summarized in three stages:

\begin{enumerate}
    \item \textbf{Manufacturing}. The $k$ policies are trained in a simulator to jointly maximize the welfare of $n$ agents in an r-MDP. Taking self-driving cars as an example, this stage involves developing $k$ driving policies based on aggregated user preferences (e.g., gathered from surveys). \textit{Our primary focus is on this stage.}
    
    \item \textbf{Assessment}. At this stage, regulatory authorities evaluate the developed policies. The costs incurred here stem from the extensive review process and potential requests for policy modifications. The number of policies, $k$, naturally balances the degree of personalization offered by each policy against the assessment costs.
    
    \item \textbf{Deployment}.  Following successful assessment, the policies are authorized for real-world deployment.
\end{enumerate}

As we discuss later, solving r-MDP directly is intractable. However, we can simplify the problem by separating it into two more manageable sub-problems: optimizing policies given fixed assignments and optimizing assignments given fixed policies. Drawing inspiration from the classic K-means \cite{macqueen1967classification,lloyd1982least} and Expectation-Maximization (EM) \cite{dempster1977maximum}  clustering algorithms, we introduce our first algorithm, which iteratively updates policies and assignments. Moreover, recognizing the differentiability of policy objectives with respect to assignments, we propose our second algorithm employing gradient descent for end-to-end training. We provide theoretical guarantees of monotonic improvement and convergence to a local maximum of social welfare using our algorithms.

Our empirical analysis encompasses Resource Gathering environment \cite{barrett2008learning} and four MuJoCo \cite{todorov2012mujoco} tasks, adapted as r-MDPs. The results consistently demonstrate that our algorithms surpass existing baselines in performance. Notably, we observe that even a limited number of policies can provide significant personalization, highlighting the efficacy of our approach.

\paragraph{Our contributions}

\begin{enumerate}
    \item \textbf{Problem Formulation}: We introduce a nuanced problem formulation in the realm of personalized RL, emphasizing the challenge posed by the resource-intensive review and authorization process for personalized policies.

    \item \textbf{Novel Setting}: We propose the r-MDP framework that addresses the need for a practical compromise between the desire for high personalization and the constraints of expedited regulatory review processes.

    \item \textbf{Efficient Algorithms}: We present two deep RL algorithms, backed by robust theoretical justifications, to approximately solve r-MDPs. These algorithms demonstrate superior performance in achieving personalized outcomes compared to approaches from existing literature.
\end{enumerate}

\paragraph{Limitations}

This study primarily addresses the challenge of training a limited number of policies for a large user base in the Manufacturing stage of our pipeline, leaving the complexities of the Assessment and Deployment stages, such as policy revisions and real-world performance stability, for future exploration. Additionally, our focus on utilitarian social welfare may inadvertently lead to uneven reward distributions between agents. We discuss potential alternatives in Section \ref{sec:problem_represented}. Finally, while we use parameter sharing to enhance sample efficiency, the possibility of further improvements through advanced techniques remains. Nevertheless, given the controlled nature of simulator training, sample efficiency is a secondary concern in our study.

\subsection{Related Work}\label{sec:introduction_related}

Related works in personalization, multiple objectives, and multi-agent systems provide valuable context for our research, yet none directly address the unique challenge of operating within an explicitly constrained policy budget.

\paragraph{RL for personalization}

This field aims at creating tailored RL solutions for individuals or groups. For an in-depth review, see \cite{den2020reinforcement}. A key challenge is personalizing RL policies in real-world applications, especially in healthcare \cite{hassouni2018personalization,zhu2018group,grua2018exploring,el2019end,el2022ph}. While offering a single policy to all users can be suboptimal, training a policy per user can be an inefficient use of collected samples. A common strategy involves clustering users by their behavior to train cluster-specific policies. Unlike these approaches where clustering is driven by sample efficiency, our approach addresses real-world policy implementation costs, with training conducted in a simulator. Still, the trajectory-clustering concept is relevant to our framework and serves as a baseline in our experiments. We also acknowledge works that use external data for clustering \cite{martin2004agentx,goindani2020cluster}, but our methods do not require such data.

Other aspects in RL for Personalization include privacy-respecting data sharing \cite{tabatabaei2018narrowing,baucum2022adapting}, which can be addressed with Federated RL \cite{nadiger2019federated}, and exploration under safety constraints \cite{perkins2002lyapunov,hans2008safe,moldovan2012safe,junges2016safety}. Though significant, these challenges do not align closely with our specific research focus.

\paragraph{Meta-RL}

While Meta-RL aims for policy adaptability to an unlimited number of tasks \cite{finn2017model}, r-MDP imposes a strict constraint on the number of policies. A capable meta-policy could offer a personalized solution to each user in the absence of such a constraint, but optimally choosing a limited subset of policies to meet the needs of all users is a unique challenge of our framework. Note that there exists a potential for synergy: Meta-RL could provide a versatile policy that our algorithms would deploy strategically within the explicit policy budget. This synergy emphasizes that the two frameworks address distinct but potentially complementary aspects of the RL problem space.

\paragraph{Multi-Objective RL (MORL)}

Similarly to our setting, MORL involves optimizing multiple rewards. However, MORL typically focuses on either developing a single policy that balances various objectives or approximating the Pareto front with a potentially large set of policies \cite{hayes2022practical}. While the latter algorithms could technically be adapted to our setting, for instance by selecting $k$ policies from the Pareto set, they only apply to problems with a few reward functions. In contrast, we tackle problems with as many as a thousand reward functions, which underscores the scalability of our framework.

\paragraph{Multi-Agent RL (MARL)}

While superficially similar, MARL differs fundamentally from our framework. MARL involves multiple agents acting and interacting within a shared environment, often formalized as a Markov game \cite{littman1994markov}. In contrast, our framework trains policies that operate in a single-agent environment independently, without inter-policy interaction. This key distinction sets our work apart from the interactive dynamics central to MARL, emphasizing our focus on individual preference optimization.

\section{Background and Problem Setup}
\label{sec:problem}

\subsection{Markov Decision Process}
\label{sec:problem_markov}

A Markov Decision Process (MDP) is a tuple $\mathcal{M} = (S, A, \mathcal{T}, \mathcal{T}_0, r, \gamma)$, where: $S$ is the set of all states $s$; $A$ is the set of all actions $a$ available to the agent; $\mathcal{T}: S \times A \rightarrow \Delta(S)$ is the transition function that specifies the distribution of next states, where $\Delta$ denotes a set of discrete probability distributions; $\mathcal{T}_0 = \Delta(S)$ specifies the distribution of initial states $s_0$; $r: S \times A \rightarrow \mathcal{P}(\mathbb{R})$ is the reward function that specifies the distribution of rewards, where $\mathcal{P}$ is a set of continuous probability distributions; $\gamma \in (0, 1)$ is the discounting factor.

Let $\pi: S \rightarrow \Delta(A)$ be a policy. Given $s \in S$, $\pi(s, a)$ denotes the probability assigned to action $a$. A transition is a tuple $(s, a, \tilde{r}, s')$, where $a \sim \pi(s)$, $\tilde{r} \sim r(s,a)$, and $s' \sim \mathcal{T}(s, a)$. An episode is a sequence of transitions, in which each transition corresponds to a time step $t = 0, 1, \dots, T$. The episode starts at time step $t = 0$ and progresses until the terminal time step $T$, which marks the horizon.

$R_t = \overset{T}{\underset{l=t}{\sum}} \left[ \gamma^{l-t} \tilde{r}_l  \right]$ is a return of an episode at time step $t$. The value function $V^\pi: S \rightarrow \mathbb{R}$ is defined as $V^\pi(s) \equiv V(s \mid \pi) = \mathbb{E} [ R_t \mid s_t = s, \pi ]$. The objective is to find the policy that maximizes the value function in all states:

\begin{equation}
\argmax_{\pi} \mathbb{E}_{\mathcal{T}, \mathcal{T}_0} V^\pi(s) = \argmax_{\pi} \mathbb{E}_{s_0 \sim \mathcal{T}_0} V^\pi(s_0).
\end{equation}

\subsection{Represented Markov Decision Process}
\label{sec:problem_represented}

We define a represented MDP (r-MDP) as a tuple $\mathcal{M}_r = (S, A, \mathcal{T}, \mathcal{T}_0, \gamma, N, K, (r^i)_{i \in N})$, where: $S, A, \mathcal{T}, \mathcal{T}_0, \gamma$ are defined above; $N$ is the set of agents, where $\left| N \right| = n$; $K$ is the set of representatives, where $\left| K \right| = k < n$ is the budget of policies; $r^i: S \times A \rightarrow \mathcal{P}(\mathbb{R})$ is a reward function of $i \in N$.

In r-MDPs, agents do not interact with the environment directly. Instead, each agent $i \in N$ is represented by one of $k$ representatives $j \in K$ that acts for them. Denote $\pi^j: S \rightarrow \Delta(A)$ as the $j$-th representative policy; and $\alpha^i \in \Delta(K)$ as the $i$-th agent's assignment, with $\alpha^i(j)$ denoting the probability of agent $i$ being represented by $j$. The objective in r-MDP is to both match agents with representatives and train the representative policies such that the utilitarian social welfare of all agents is maximized:

\begin{equation}\label{eq:objective_joint}
    \max_{(\alpha^i)_{i \in N}, (\pi^{j})_{j \in K}} \sum_{i, j} \alpha^i(j) \mathbb{E}_{\mathcal{T}_0} V^{ij}(s_0),
\end{equation}

\noindent where $V^{ij}(s) = V^i(s \mid \pi^j) = \mathbb{E} [ R_t^i \mid s_t = s, \pi^j ]$ is the value function of agent $i$ assigned to representative $j$.

Note that representatives are an abstraction to distinguish the actors in the environment and the agents. In particular, representatives do not have intrinsic reward functions and maximize the assigned agents' welfare. Each representative effectively interacts with its copy of the environment with identical dynamics but different reward functions (see the definition of $M^j$ in Section \ref{sec:our_factorized}). 

\paragraph{Applications} The development of represented Markov Decision Processes (r-MDPs) primarily addresses the challenge of designing personalized ML solutions subject to rigorous regulatory assessments, as highlighted in the introduction. Beyond this primary motivation, r-MDPs hold broader applicability in scenarios where solution quantity is constrained.

Take, for instance, a financial institution formulating portfolios for multiple Exchange-Traded Funds (ETFs). The institution aims to cater to a diverse range of investor preferences, such as risk tolerance, asset types, and market exposure, informed by market data or surveys. However, offering a unique portfolio to each investor is impractical, necessitating a compromise on the number of ETFs. This scenario is algorithmically solved for one-dimensional preferences \cite{diana2021algorithms}. In more complex, multi-dimensional cases, r-MDPs offer a viable modeling approach. By leveraging our r-MDP framework and the associated algorithms, the financial institution can optimally balance the diversity of ETF offerings with the practical limitations on the number of available portfolios, demonstrating the adaptability and utility of r-MDPs in varied contexts beyond regulatory constraints.

\paragraph{Limitations} Optimizing utilitarian social welfare may result in unfair reward distributions between agents. Egalitarian or Nash-product social welfare may be more reasonable in applications where this is a concern. To this end, the techniques from socially fair RL \cite{mandal2022socially} and clustering \cite{kar2023feature} could potentially be adapted. However, this direction is out of the scope of this paper.

\subsection{Proximal Policy Optimization}
\label{sec:problem_proximal}

PPO \cite{schulman2017proximal} is a deep RL algorithm based on the prominent Actor-Critic framework. 

The critic is a neural network $\phi$ that parameterizes an approximation of the agent's value function $\tilde{V}(s)$. It is trained with gradient descent to minimize the mean squared difference with a target value:

\begin{equation}\label{eq:loss_ppo_critic}
    L(\phi) = \sum_{t \in B} \left[\tilde{A}_\phi(s_t, a_t) = (y(s_t, a_t) - \tilde{V}_\phi(s_t)) \right]^2,
\end{equation}

\noindent where $L$ denotes a loss function, $B$ denotes a batch of transitions, $y$ denotes a target value, and $\tilde{A}_\phi(s_t, a_t)$ denotes an approximation of the advantage, which we estimate using generalized advantage estimation \cite{schulman2015high}. 

The actor is a neural network $\theta$ that parameterizes the policy $\pi$. It is trained to minimize the negated clipped surrogate objective:

\begin{equation}\label{eq:loss_ppo_actor_short}
    L(\theta) = - \sum_t \clip(\rho_\theta(s_t, a_t)) \tilde{A}(s_t, a_t).
\end{equation}

\noindent where $\rho_\theta(s, a) = \frac{\pi_\theta(s, a)}{\pi_{old}(s, a)}$ is a ratio between the policy that is being optimized and the policy that collected the experience, and $\clip(\cdot)$ truncates the argument according to a specific rule that we report in the Appendix. Repeatedly updating on this objective using a batch of experience divided into smaller mini-batches approximates a policy update within a trust region \cite{schulman2015trust}.

Multiple technical details can make or break an implementation of PPO. For our experiments, we relied on \cite{shengyi2022the37implementation} and were able to replicate the performance reported in the original PPO paper.

\section{Our Approach and Algorithms}
\label{sec:our}

In this section, we describe our factorized approach to r-MDPs and propose two factorized deep RL algorithms.

\subsection{Factorized Approach}
\label{sec:our_factorized}

Directly optimizing (\ref{eq:objective_joint}) involves finding the optimal joint assignment from a set exponential in the number of agents. Even if restricted to discrete assignments, the cardinality of this set is $K^n$, making the problem intractable for large $n$.

Consider a simplification of the joint objective (\ref{eq:objective_joint}) where the policies $\pi^j$ are fixed for all $j \in K$. Then, maximizing it reduces to independently solving a set of trivial problems:

\begin{equation}\label{eq:objective_assignment}
    \max_{\alpha^i} \left[ V^i =  \sum_j \alpha^i(j) \mathbb{E}_{\mathcal{T}_0} V^{ij}(s_0) \right],
\end{equation}

The optimal solution is to greedily assign agent $i$ to the best-performing representative $j^*$ (assuming its uniqueness):

\begin{equation}\label{eq:optimal_assignment}
    \alpha^i(j^*) = 1 \iff j^* = \argmax_j Q^i(j),
\end{equation}

\noindent where $Q^i(j) = \mathbb{E}_{\mathcal{T}_0} V^{ij}(s_0)$ is the Q-value of assigning $i$ to $j$. Ties can be broken arbitrarily. This quantity can be empirically approximated with Monte-Carlo sampling for each $(i, j)$ as an average welfare over several episodes.

Consider another simplification of (\ref{eq:objective_joint}) where the assignments $\alpha^i$ are fixed for all $i \in N$. Then, optimizing social welfare over the policies reduces to independently solving a set of MDPs $(\mathcal{M}^j = (S, A, \mathcal{T}, \mathcal{T}_0, r^j, \gamma))_{j \in K}$, where $r^j(s, a) = \sum_i \alpha^i(j) r^i(s, a)$.  The objective in $\mathcal{M}^j$ is:

\begin{equation}\label{eq:objective_policy}
    \max_{\pi^{j}} \Bigl[ \mathbb{E}_{\mathcal{T}_0} V^j(s_0) = \sum_i \alpha^i(j) \mathbb{E}_{\mathcal{T}_0} V^{ij}(s_0) \Bigr],
\end{equation}

\noindent where $V^j(s)$ is the value of policy $\pi^j$ in state $s$.

We define the factorized approach as an independent optimization of objectives (\ref{eq:objective_assignment}) and (\ref{eq:objective_policy}). That is, each assignment is myopically optimized given the current policies, and vice versa. We are interested in designing factorized algorithms that approximate optimal solutions to the joint objective (\ref{eq:objective_joint}).

\subsection{Training the Representatives}\label{sec:our_representatives}

Before describing our algorithms that simultaneously train assignments and policies, we focus on the latter given fixed assignments. As the backbone, we use the PPO algorithm described in Section \ref{sec:problem_proximal}, but note that any other Actor-Critic algorithm would suffice.

Recall that a representative optimizes the expected value function defined in (\ref{eq:objective_policy}). Estimating it requires the summation of values $V^{ij}(s)$ over all agents weighted by the assignment probabilities $\alpha^i(j)$, which change throughout training. Because of this, directly parameterizing $\tilde{V}^j(s)$ with a neural network $\phi^j$ (as a direct application of the Actor-Critic approach would suggest) results in a non-stationary objective for the critic. Instead, we parameterize $\tilde{V}^{ij}(s)$. Specifically, a critic parameterized with $\phi^j$ outputs $n$ values $\tilde{V}_{\phi^j}^{ij}(s)$, representing the welfare of each agent when assigned to $j$. Each output is trained to minimize the loss function (\ref{eq:loss_ppo_critic}) given rewards sampled from $r^i(s, a)$ and actions sampled from $\pi^j(s)$: 

\begin{equation}\label{eq:loss_our_critic}
    L(\phi^j) = \sum_{i, t} \left( \tilde{A}_{\phi^j}^{ij}(s_t, a_t) \right)^2,
\end{equation}

\noindent where $\tilde{A}_{\phi^j}^{ij}(s_t, a_t) = (y^{ij}(s_t, a_t) - \tilde{V}^{ij}_{\phi^j}(s_t))$. The marginal advantage $\tilde{A}_{\phi^j}^j(s, a) = \sum_i \alpha^i(j) \tilde{A}_{\phi^j}^{ij}(s, a)$ is estimated according to the current assignments. Then, an actor $\theta^j$ that parameterizes $\pi^j$ can be trained on the objective (\ref{eq:loss_ppo_actor_short}):

\begin{equation}\label{eq:loss_our_actor}
    L(\theta^j) = - \sum_t \clip(\rho_{\theta^j}^j(s_t, a_t)) \tilde{A}_{\phi^j}^j(s, a).
\end{equation}

\noindent To improve training efficiency, we share the parameters of intermediate layers between actors $\theta^j$, as well as critics $\phi^j$.

Note that training the policies requires the experiences of all representatives acting for all agents. However, since the dynamics are identical, performing one transition with a representative can be used to sample rewards for all agents.

\subsection{Hard Assignment via EM-like Learning}\label{sec:our_hard}

Our EM-like algorithm is inspired by the classic K-means \cite{macqueen1967classification,lloyd1982least} and Expectation-Maximization (EM) \cite{dempster1977maximum} clustering algorithms. It alternates between two steps. At the E-step, agents are assigned to representatives in analogy to points being assigned to clusters. At the M-step, representatives' policies are improved given the assignments of agents in analogy to cluster centers being improved given the assignments of points.

We maintain an $n \times k$ table $\tilde{Q}$, each element of which $\tilde{Q}^{ij}$ approximates the corresponding Q-value $Q^i(j)$ (defined in Section \ref{sec:our_factorized}). Before performing the E-step, the elements of table $\tilde{Q}$ are updated as moving averages:

\begin{equation}\label{eq:hard_Q_table}
    \tilde{Q}^{ij} \leftarrow (1 - \lambda) \tilde{Q}^{ij} + \lambda [R_0^i \sim \pi^j],
\end{equation}

\noindent where $\lambda \in (0, 1]$ is a mixing coefficient. Technically, for each assignment $\alpha^i$, this update rule is Q-learning with learning rate $\lambda$ in a stateless environment, albeit non-stationary since policies change over time. At the E-step, each agent is greedily reassigned to a best-performing representative, approximating (\ref{eq:optimal_assignment}):

\begin{equation}\label{eq:hard_E_step}
    \alpha^i(j^*) = 1 \iff j^* = \argmax_j \tilde{Q}^{ij}.
\end{equation}

At the M-step, we update policies with PPO as described in Section \ref{sec:our_representatives}. A crucial trade-off is that of the frequency of E-steps and the magnitude of M-steps. We found it best to perform an E-step as frequently as possible, resulting in an M-step that corresponds to a single PPO update per policy.

Similarly to K-means, our algorithm can be proven to converge to a local optimum. We formulate this as a theorem:

\begin{theorem}\label{theorem}
    Given an r-MDP, the EM-like algorithm converges to a local maximum of utilitarian social welfare. 
\end{theorem}

\noindent The proof is provided in the Appendix. Assuming that the M-step is performed until convergence with an (RL) algorithm with global convergence guarantees, we show that both the E-step and M-step monotonically improve social welfare.

\subsection{Soft Assignment via End-to-End Learning}\label{sec:our_soft}

Our second algorithm is based on an observation that the loss function of a representative policy (\ref{eq:loss_our_actor}) is differentiable with respect to the assignment probabilities $\alpha^i(j)$. We leverage this by parameterizing assignments $\alpha^i$ for all agents with $\psi$ and updating the parameters to minimize the same loss function as the policies. The resulting loss function for $\psi$ is:


\begin{equation}\label{eq:soft_loss}
    L(\psi) = - \sum_{j, t} \clip(\rho_{\theta^j}^j(s_t, a_t)) \sum_i \alpha_{\psi}^i(j) \tilde{A}_{\phi^j}^{ij}(s_t, a_t).
\end{equation}

This parameterization is implemented as an $n \times k$ table $\psi$ of logits that are transformed into $\alpha_\psi^i(j)$ by applying column-wise softmax function so that $\sum_j \alpha_\psi^i(j) = 1$. These logits can be updated in the same backward pass as the actors $\theta^j$ since they share the loss function.

The intuition behind this update rule is that the probability $\alpha_\psi^i(j)$ only increases for the best-performing representative, i.e., such $j$ that maximizes the advantage $\tilde{A}^{ij}(s, a)$ averaged over the mini-batch. Effectively, this is a relaxation of the hard re-assignment (\ref{eq:hard_E_step}) of our EM-like algorithm.

\section{Experiments}
\label{sec:experiments}

As we move into the experimental phase of our study, we first describe the environments selected for testing our algorithms, as well as the baselines used for comparison. This is followed by an in-depth analysis of the experimental results, demonstrating the performance of our algorithms in diverse scenarios. Through this, we aim to substantiate the theoretical aspects of our work with empirical evidence, highlighting the strengths and limitations of our approach.\footnote{Code: \url{https://github.com/dimonenka/RL_policy_budget}}

\paragraph{Environments}

To evaluate our algorithms, we employ two distinct types of environments, each serving a specific purpose in our study.

Our initial objective is to scrutinize the behavior of our algorithms in a controlled, simplified setting. We use the Resource Gathering environment adapted from \cite{barrett2008learning,Alegre+2022bnaic}, where a policy directs a character in a 5x5 grid world to collect resources. In our r-MDP modification, each of the $n=25$ unique agents is assigned a specific resource tile. The goal is to collect these resources efficiently, with the episode ending when the character returns to the starting tile. The agents’ rewards for collecting the respective resources, calculated as $r = 100 - T$, incentivize quick resource collection and require optimal pathfinding. In this scenario, we explore policy budgets ranging from $k \in \{1,2,3,5,10,25\}$, examining how the number of policies affects efficiency and agent satisfaction.

To rigorously test our algorithms in more complex scenarios, we employ MuJoCo environments \cite{todorov2012mujoco,tassa2018deepmind,tunyasuvunakool2020}, including HalfCheetah, Ant, Hopper, and Walker2d. These tasks involve controlling robots with continuous actions in high-dimensional states. Each episode lasts for $1000$ time steps or until the robot falls. To adapt these environments as r-MDPs, we define $n \in \{100, 1000\}$ agents, and for each agent, uniformly sample a target velocity $v^i \sim U[0, b]$, where $b$ is selected as $2.5$ for Walker2d and Hopper, $3$ for Ant, and $4$ for HalfCheetah. This is inspired by the meta-RL literature, where each sampled velocity is treated as a different task \cite{finn2017model}.  Each time step, the agents are rewarded for the proximity of the robot to their target velocities according to the reward function $r^i(s_t, a_t) = 1 - \min(1, 20 \cdot \left| v^i - v_{t+1} \right| / b)$. The rewards are normalized s.t. the cumulative reward over an episode equals $100$ for an agent when its target speed is maintained perfectly. Note that the reward of a particular agent is non-zero in only a narrow velocity interval, which echoes the costly error scenarios depicted in our introductory examples. We experiment with policy budgets $k \in \{1, 2, 5, 10, 50\}$.

Both environments offer distinct challenges: the Resource Gathering environment tests the algorithms' effectiveness in a discrete, straightforward setting, while the MuJoCo tasks present a more complex and continuous challenge. Together, they comprehensively evaluate our algorithms’ ability to handle diverse agent preferences and policy budget constraints, reflecting the scenarios discussed in our introduction.

\begin{figure}[t]
\begin{center}
    \begin{subfigure}{0.5\linewidth}
        \centering
        \caption{An optimal path, k=1}
        \includegraphics[width=\linewidth]{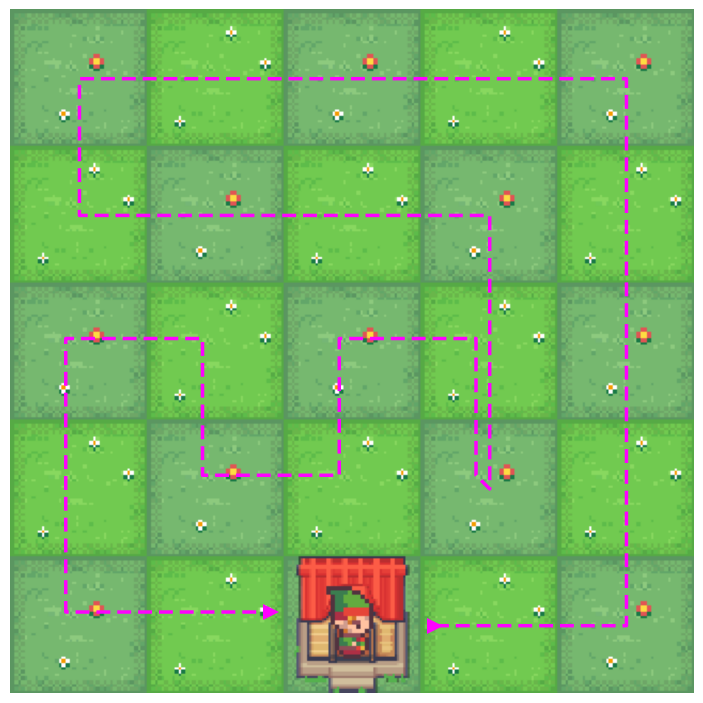}
        \label{fig:results_resource_k=1}
    \end{subfigure}%
    \hfill%
    \begin{subfigure}{0.5\linewidth}
        \centering
        \caption{Paths of EM, k=2}
        \includegraphics[width=\linewidth]{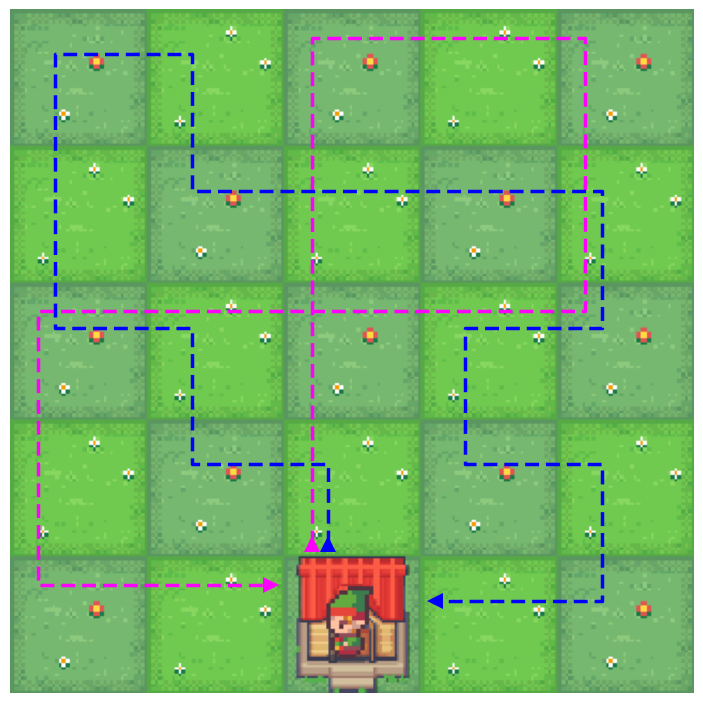}
        \label{fig:results_resource_k=2}
    \end{subfigure}

    \begin{subfigure}{0.5\linewidth}
        \centering
        \caption{Paths of EM, k=3}
        \includegraphics[width=\linewidth]{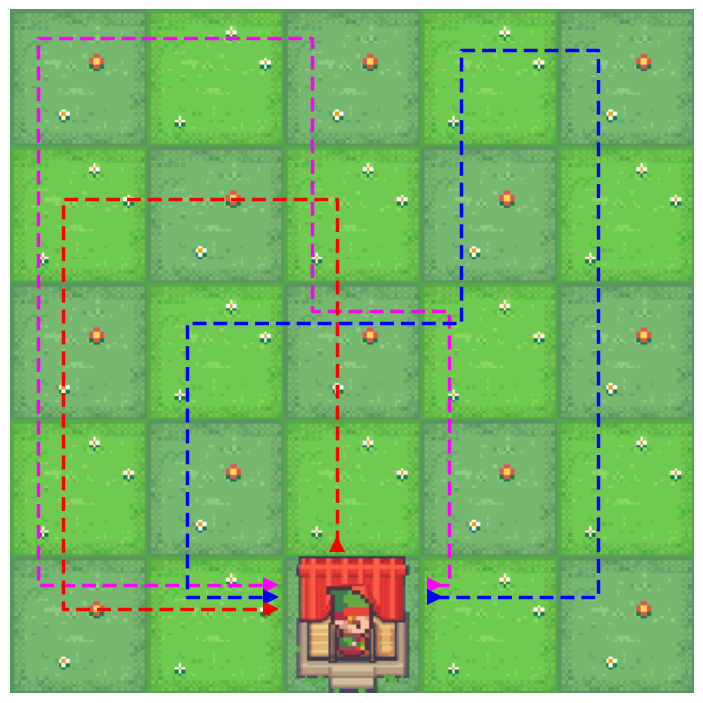}
        \label{fig:results_resource_k=3}
    \end{subfigure}%
    \hfill%
    \begin{subfigure}{0.5\linewidth}
        \centering
        \caption{Paths of EM, k=5}
        \includegraphics[width=\linewidth]{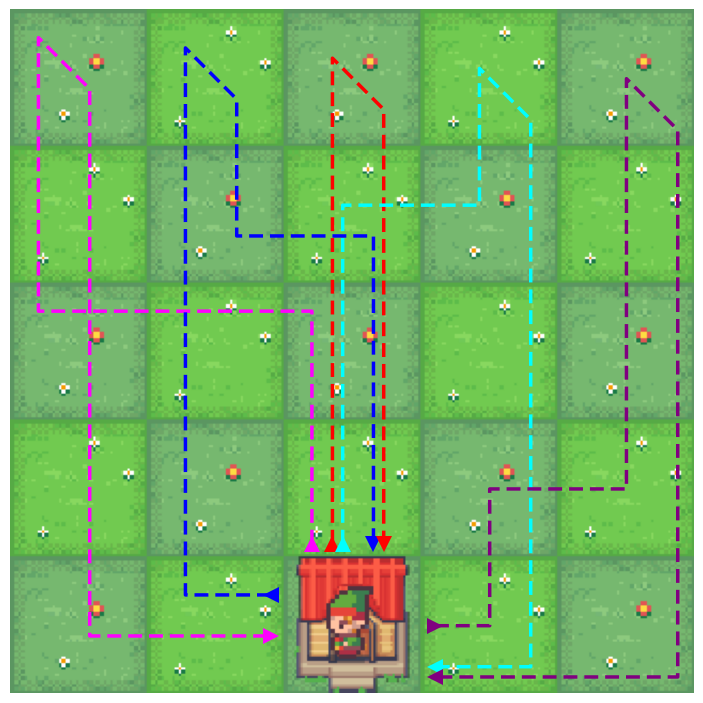}
        \label{fig:results_resource_k=5}
    \end{subfigure}
    
\caption{Paths that representatives learn in Resource Gathering after being trained with our EM algorithm for different $k$ and $n=25$ ($0$-th random seed). The representatives divide the map such that 1) each tile is visited by some policy and 2) policies jointly minimize the average episode length.}
\label{fig:results_resource}
\end{center}
\end{figure}

\begin{figure}[t]
\begin{center}
    \centering
    \includegraphics[width=0.8\linewidth]{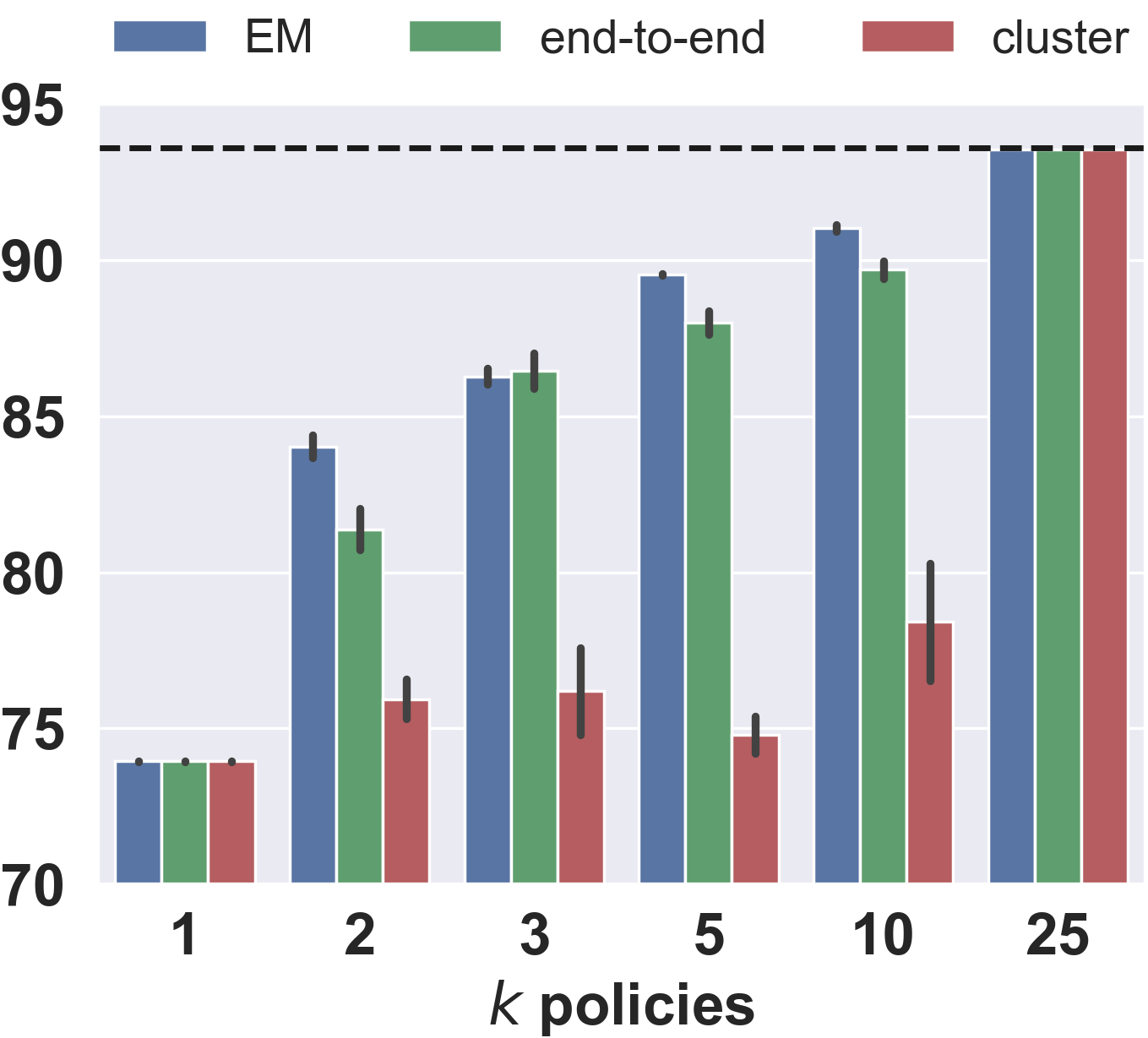}
\caption{Performance of ours and baseline algorithms in Resource Gathering for different $k$ and $n=25$. The black dashed line represents the optimum for $k=25$. For each $k$, all algorithms are trained for 1 million transitions per policy. For $k=1$, all algorithms reduce to solving an MDP with a single policy. Confidence intervals represent standard errors.}
\label{fig:results_resource_sw}
\end{center}
\end{figure}

\paragraph{Algorithms}

In our experiments, we utilize two algorithms developed in this study, referred to as the EM algorithm and the end-to-end algorithm, as detailed in Sections \ref{sec:our_hard} and \ref{sec:our_soft}.

Finding suitable baselines for comparison proved challenging due to the unique constraints of the r-MDP framework, which are not addressed by most methods in related fields. However, we identified a relevant baseline in a clustering-based algorithm used in RL for personalization in healthcare (see Section \ref{sec:introduction_related}). This algorithm typically pre-trains a universal policy for all agents, employs K-Means clustering on sampled trajectories to group agents, and then trains a policy for each cluster. To align this method with our r-MDP setting, where sample efficiency is less of a concern, we extend both the pre-training of the universal policy and the training of cluster-specific policies to approximate convergence.

As a control, we also include a weak baseline where agents are randomly assigned to representatives, with policies subsequently trained without personalization considerations. This serves to benchmark the minimum expected performance and to emphasize the impact of personalized approaches.

To ensure reproducibility and transparency, hyperparameters and technical details are provided in the Appendix. All experiments were conducted ten times to ensure robustness, with standard errors reported alongside mean values.

\subsection{Resource Gathering}
\label{sec:experiments_resource}

\begin{figure*}[t]
\begin{center}

    \centering
    \begin{subfigure}{.55\linewidth}
        \centering
        \includegraphics[width=\linewidth]{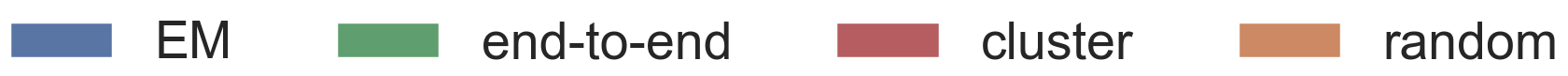}
    \end{subfigure}
    
    \begin{subfigure}{0.245\textwidth}
        \centering
        \caption{Ant, social welfare}
        \includegraphics[width=\linewidth]{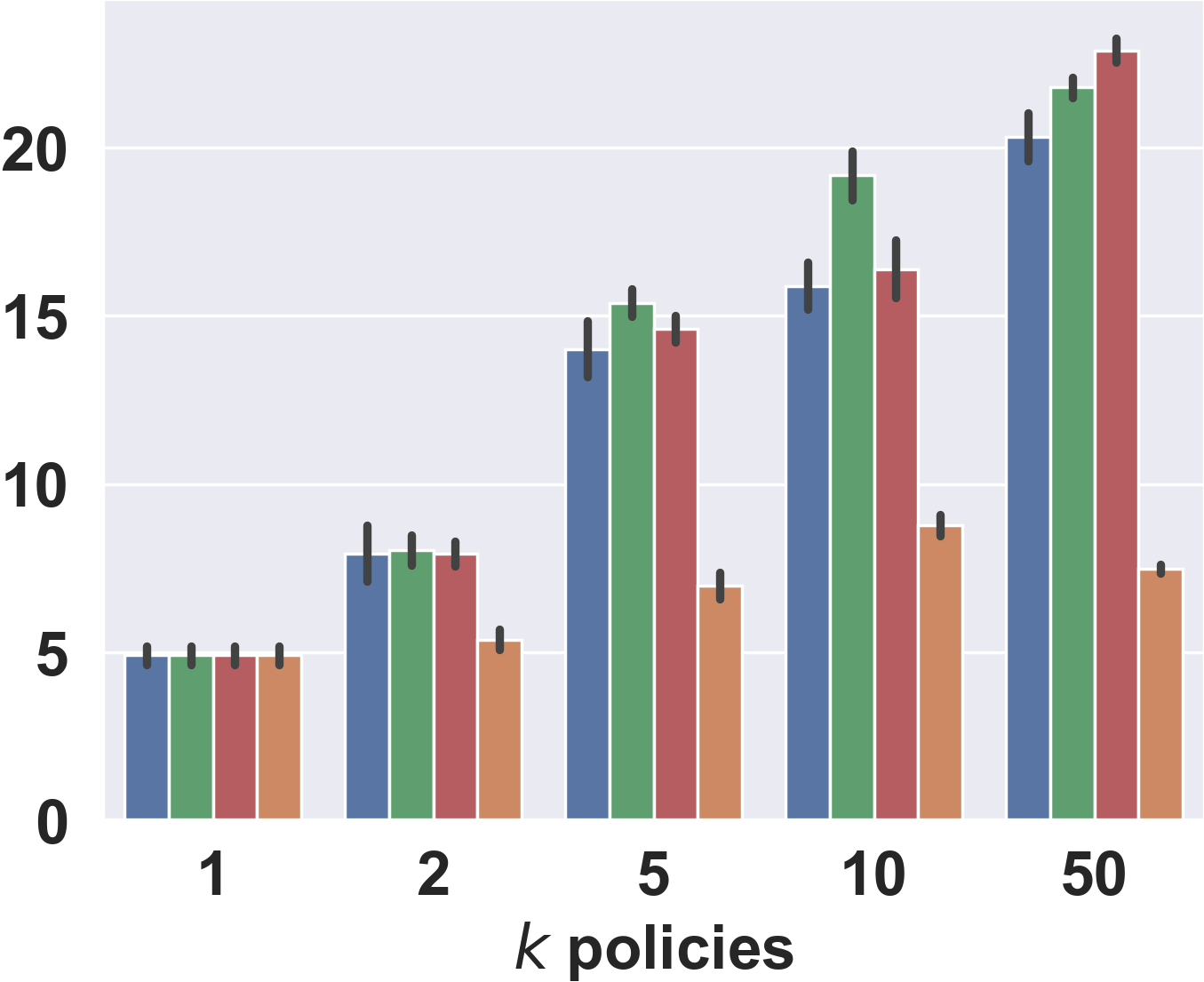}
    \end{subfigure}%
    \hfill%
    \begin{subfigure}{0.245\textwidth}
        \centering
        \caption{HalfCheetah, social welfare}
        \includegraphics[width=\linewidth]{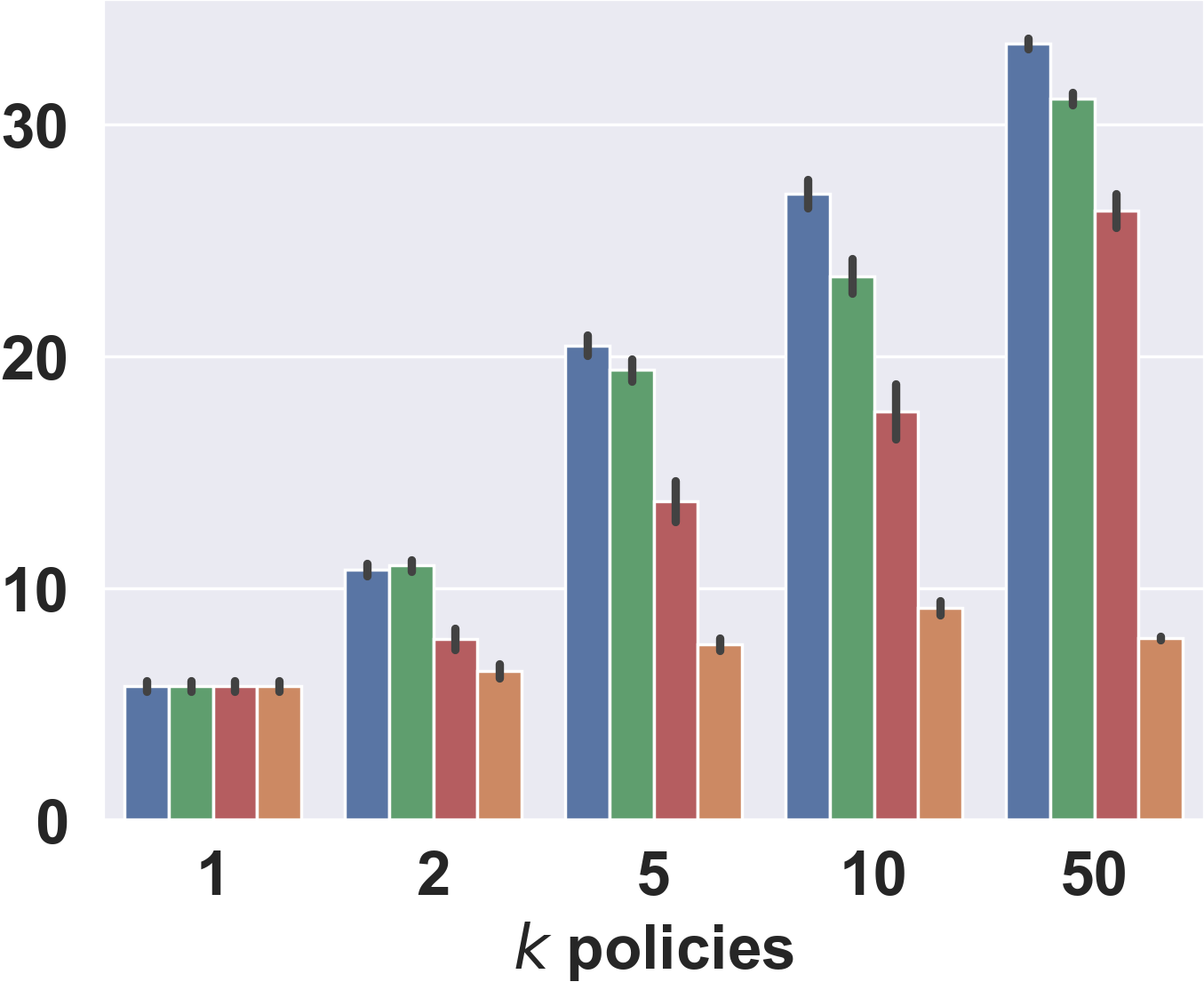}
    \end{subfigure}
    \hfill%
    \begin{subfigure}{0.245\textwidth}
        \centering
        \caption{Hopper, social welfare}
        \includegraphics[width=\linewidth]{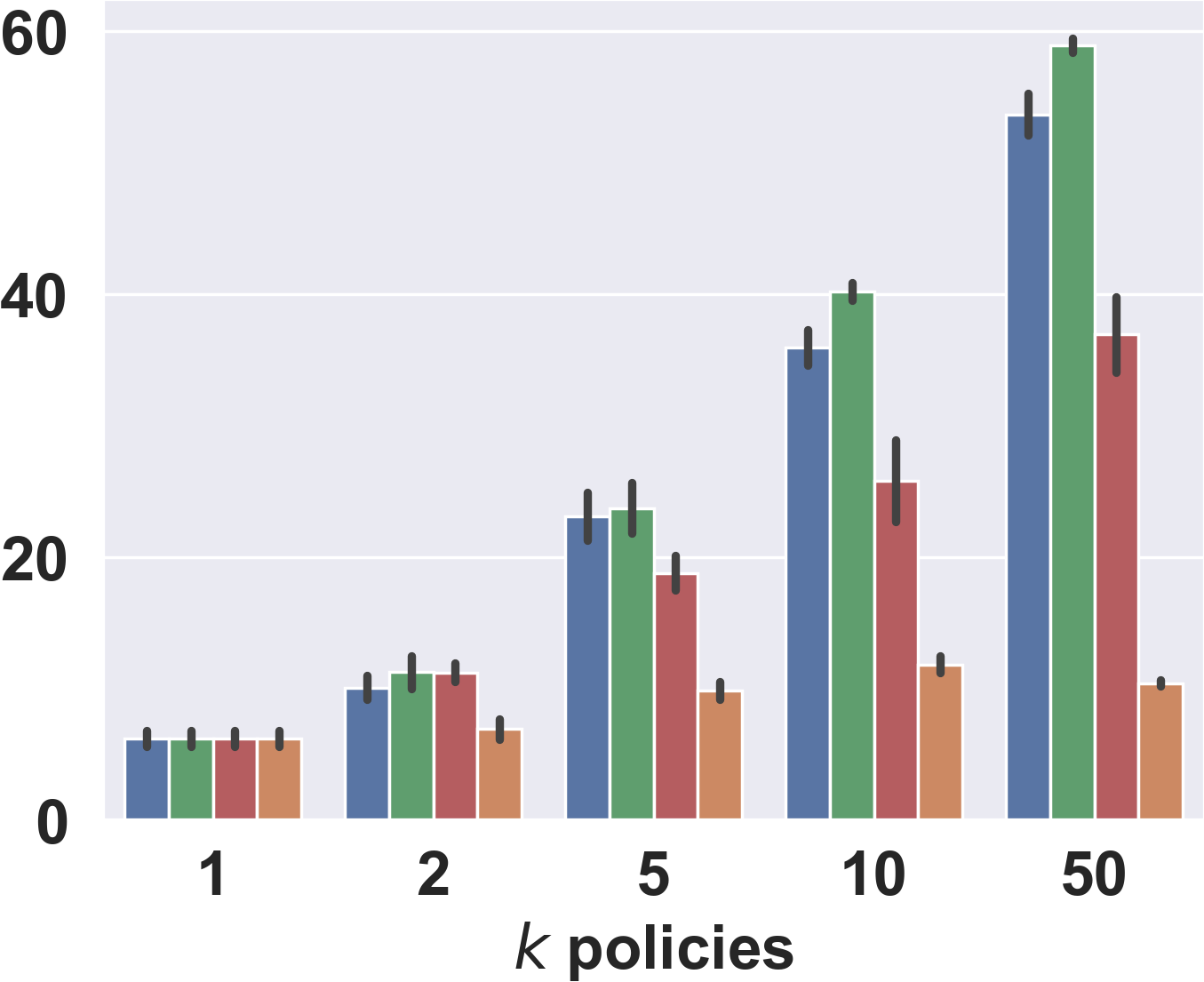}
    \end{subfigure}%
    \hfill%
    \begin{subfigure}{0.245\textwidth}
        \centering
        \caption{Walker2d, social welfare}
        \includegraphics[width=\linewidth]{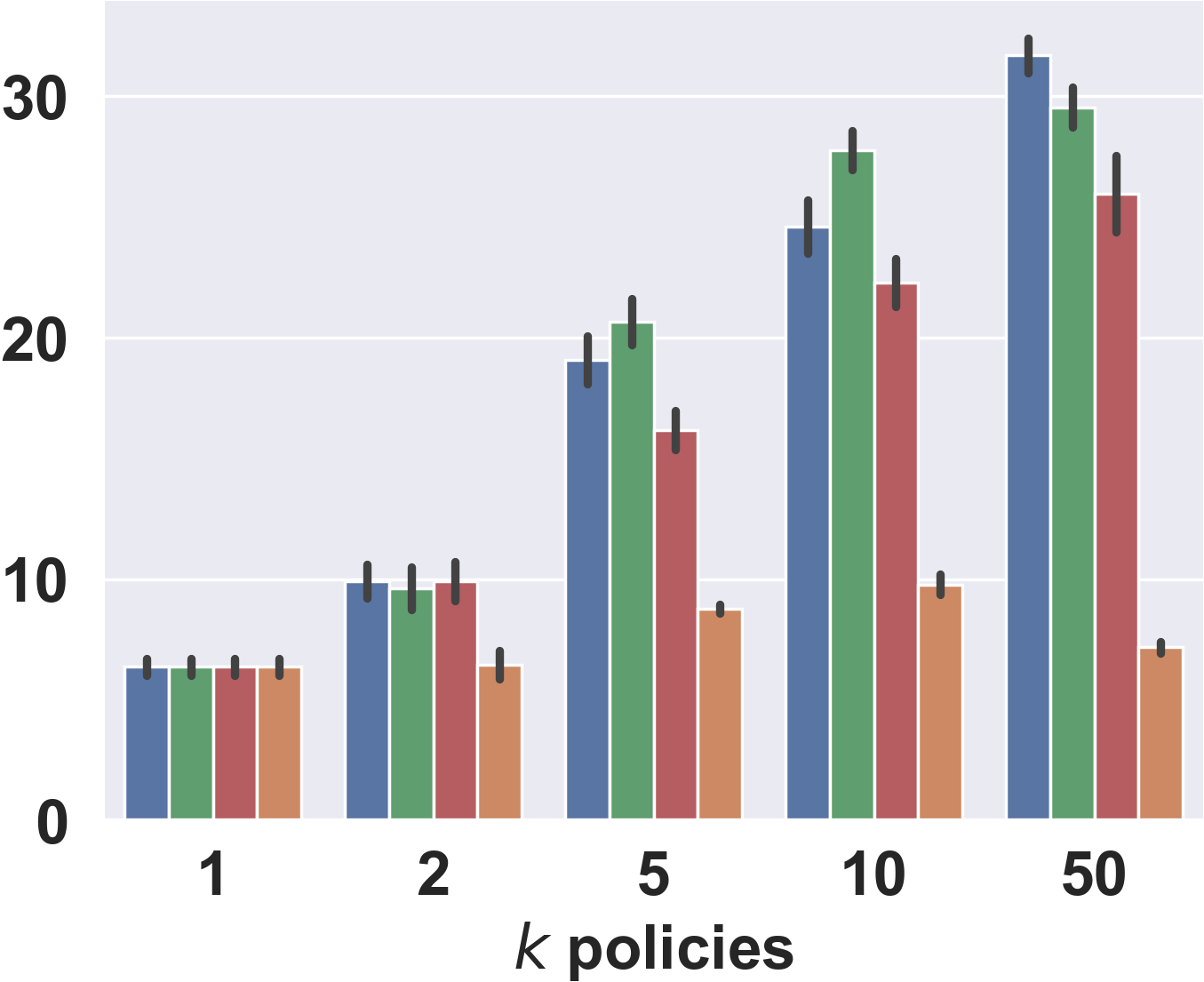}
    \end{subfigure}
    
\caption{Performance of ours and baseline algorithms in MuJoCo environments. For each $k$, all algorithms are trained for 2 million transitions per policy. The number of agents is $n=1000$ for $k=50$ and $n=100$ for smaller $k$. For $k=1$, all algorithms reduce to solving an MDP with a single policy. Confidence intervals represent standard errors.}
\label{fig:results_mujoco}
\end{center}
\end{figure*}

\begin{figure*}[t]
\begin{center}
    \begin{subfigure}{0.3\textwidth}
        \centering
        \caption{HalfCheetah, EM (ours)}
        \includegraphics[width=\linewidth]{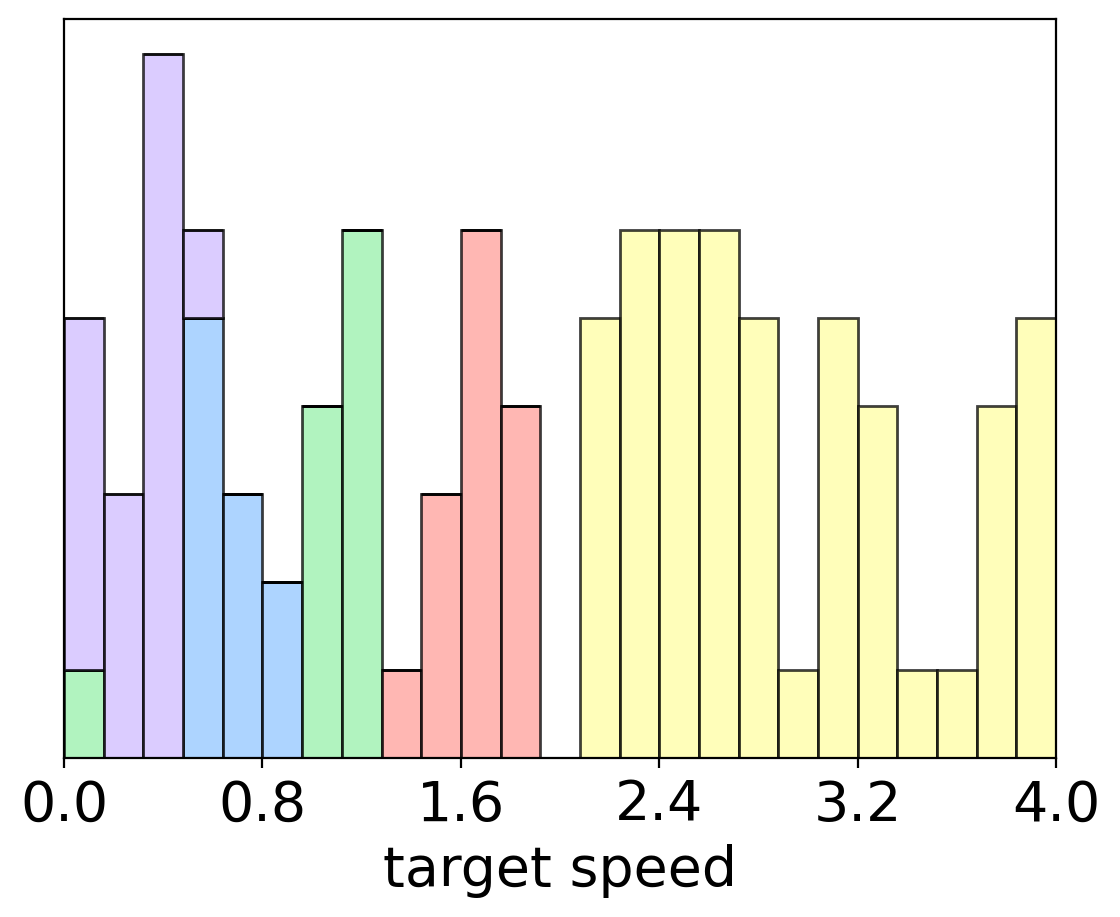}
    \end{subfigure}\hspace{10pt}%
    \begin{subfigure}{0.3\textwidth}
        \centering
        \caption{HalfCheetah, end-to-end (ours)}
        \includegraphics[width=\linewidth]{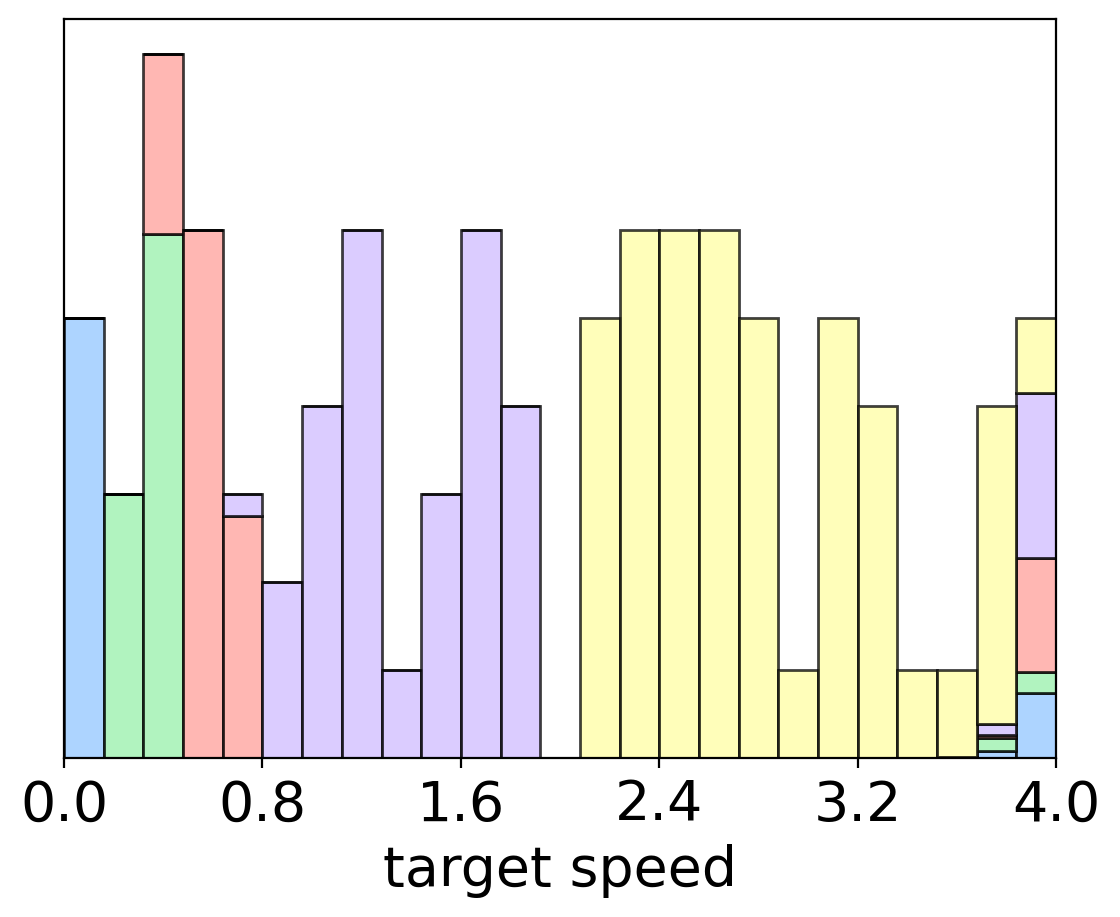}
    \end{subfigure}\hspace{10pt}%
    \begin{subfigure}{0.3\textwidth}
        \centering
        \caption{HalfCheetah, cluster (baseline)}
        \includegraphics[width=\linewidth]{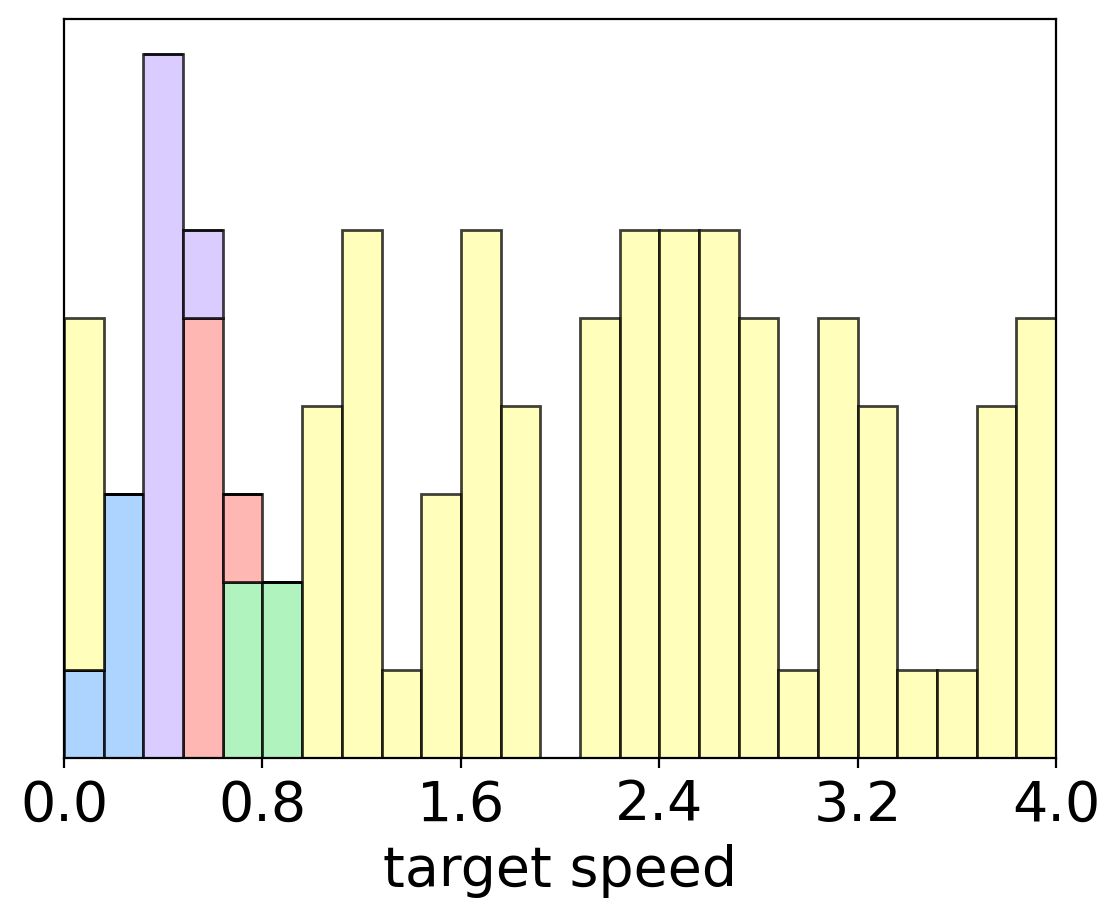}
    \end{subfigure}
\caption{Histograms of agent assignments learned by ours and baseline algorithms for $n=100$, $k=5$ in HalfCheetah ($0$-th random seed). Each color denotes one of five representatives and bars of this color denote the target velocities of agents assigned to this representative. The expected behavior is a division of the agents' velocities into five intervals of similar sizes, one for each representative. Histograms for other environments are reported in the Appendix.}
\label{fig:results_histograms}
\end{center}
\end{figure*}

The results, as depicted in Figure \ref{fig:results_resource_sw}, showcase the effectiveness of our EM and end-to-end algorithms in comparison to the clustering baseline across different values of  $k$. Notably, the performances $k=1$ and $k=n$ are intentionally identical for all algorithms, as these scenarios either involve a single representative for all agents or individual representatives for each, eliminating the need for assignment learning.

For intermediate values of $k$, our algorithms perform similarly. Remarkably, while the optimal social welfare is $93.6$ for $k=25$, our algorithms attain social welfare above $90$ with just $k=10$ policies and above $85$ with as few as $k=3$ policies. This efficiency is illustrated through the representatives' paths in Figures \ref{fig:results_resource_k=2}-\ref{fig:results_resource_k=5}, where they efficiently divide the map to cover smaller areas more rapidly, showcasing our approach's efficacy in achieving meaningful personalization under a strict policy budget. 

In contrast, the clustering baseline exhibits minimal personalization, with its performance remaining largely unchanged regardless of $k$. This limitation is evident when analyzing the unanimous policy (Figure \ref{fig:results_resource_k=1}), which traverses all tiles and thus yields identical rewards for all agents, providing no informative data for effective clustering.

These findings underscore the qualitative superiority of our algorithms over the clustering baseline. Our methods excel by learning assignments that directly optimize social welfare, in contrast to the baseline's reliance on a heuristic unaligned with the primary task. While diversifying the behaviors of the unanimous policy might enhance the baseline's performance, such an approach would still be heuristic and exceed the scope of existing literature, thus not qualifying as a conventional baseline.

\subsection{MuJoCo Environments}
\label{sec:experiments_mujoco}

In the MuJoCo environments, our algorithms consistently outperform the baselines across various policy budgets $k$, as shown in Figure \ref{fig:results_mujoco}. Both the EM and end-to-end algorithms demonstrate high levels of performance and significantly outperform random assignments in all tested scenarios. While the clustering baseline improves upon random assignments as well, it generally falls short of the performance achieved by our algorithms, with the exception of the Ant environment.

A deeper analysis of the assignments learned by the different algorithms (Figure \ref{fig:results_histograms}) reveals intriguing patterns. Both our EM and end-to-end algorithms tend to group agents with similar target velocities and maintain relatively balanced group sizes. This suggests that they effectively identify cutoff points in the latent target velocity space for agent assignment. Notably, the end-to-end algorithm does not rigidly assign agents with the highest target velocities to any single representative, likely due to the absence of a strong enough learning signal from any representative for these high-velocity agents.

In contrast, the clustering baseline primarily segregates agents with low target velocities and lumps the majority into a single cluster. This pattern aligns with the baseline's heuristic nature, which focuses less on optimizing social welfare and more on simplistic clustering based on sampled trajectories.

\section{Conclusion}
\label{sec:conclusion}

In this study, we addressed the significant challenge of personalizing solutions in domains where regulatory assessments impose high costs of implementation. Based on the formalism of represented Markov Decision Processes (r-MDPs), we developed two deep reinforcement learning algorithms and theoretically validated their monotonic convergence to local optima. Empirically, our results underscored the efficacy of these algorithms in delivering meaningful personalization under policy budget constraints.

While our research represents a substantial stride forward, it also opens several avenues for further investigation. An important future direction is the refinement of social welfare functions to integrate fairness more comprehensively, ensuring that personalization does not come at the cost of equity. Additionally, exploring the incorporation of outside options that guarantee a minimal level of welfare for all individuals is crucial. Our current experiments, primarily centered on simulated tasks, set the stage for applying our methodology to real-world scenarios. Extending our approach to practical applications will be instrumental in verifying its effectiveness in diverse settings where personalization is key.

Overall, our work contributes to personalized reinforcement learning by addressing the dual challenges of regulatory compliance and maintaining a high level of personalization. We hope that our framework and algorithms will inspire future research in this domain and facilitate the practical deployment of personalized solutions in various complex and critical environments.

\bibliography{main}

\input{main_appendix}

\end{document}

%% file: main_appendix.tex
\appendix

\section{Convergence of the EM-like Algorithm}

We use notations from Section 2 of the main text.

The objective is to maximize utilitarian social welfare:

$$SW(\boldsymbol\alpha, \boldsymbol\pi) = \mathbb{E}_{\mathcal{T}_0} \sum_{i, j} \alpha^i(j) V^{ij}(s_0),$$

\noindent where $\boldsymbol\alpha = (\alpha^i)_{i \in N}$ and $\boldsymbol\pi = (\pi^{j})_{j \in K}$.







\begin{lemma}\label{lemma:app}
    A function $V^j$ defined by $V^j(s) = \sum_{i \in N} \alpha^i(j) V^{ij}(s)$ is a value function.
\end{lemma}

\begin{proof}
    Let $s \in S$.
    
    Apply the definition of $V^{ij}(s)$:
    
    $$
    V^j(s) = \sum_{i \in N} \alpha^i(j) \mathbb{E} \left[ \overset{T}{\underset{t=0}{\sum}} \gamma^{t} \tilde{r}^i_t \mid s_0 = s, \pi^j \right].
    $$

    Rearrange the terms:

    $$
    V^j(s) = \mathbb{E} \left[ \overset{T}{\underset{t=0}{\sum}} \gamma^{t} \sum_{i \in N} \alpha^i(j) \tilde{r}^i_t \mid s_0 = s, \pi^j \right].
    $$

    Substitute $\tilde{r}^j_t = \sum_{i \in N} \alpha^i(j) \tilde{r}^i_t$:

    $$
    V^j(s) = \mathbb{E} \left[ \overset{T}{\underset{t=0}{\sum}} \gamma^{t} \tilde{r}^j_t \mid s_0 = s, \pi^j \right].
    $$

    Observe that $V^j(s)$ is a value function by definition.
\end{proof}


Define the \textit{E-step} as updating the assignments to $\boldsymbol\alpha^*$ given the policies $\boldsymbol\pi$:

\begin{equation}
    \alpha^{i*}(j^*) = 
    \begin{dcases}
        1,& j^* = \argmax_j \mathbb{E}_{\mathcal{T}_0} V^{ij}(s_0)\\
        0,              & \text{otherwise}
    \end{dcases}
\end{equation}

\noindent Note: ties are broken arbitrarily, e.g., lexicographically.

Define the \textit{M-step} as updating the policies to $\boldsymbol\pi^*$ given the assignments $\boldsymbol\alpha$:

\begin{equation}
   \forall \{ j \mid \sum_i \alpha^i(j) > 0 \} : \pi^{j*} = \argmax_{\pi^j} \mathbb{E}_{\mathcal{T}_0} V^j(s_0)
\end{equation}

\noindent Note: if some representative $j$ is not assigned any agents after an E-step (i.e., $\sum_i \alpha^i(j) = 0$), we may assign a random agent to this representative prior to the M-step. After the M-step, the representative will implement the optimal policy for this agent. For the formal proof, this is not required.

By Lemma \ref{lemma:app}, we can use any RL algorithm that has convergence guarantees to perform the M-step for each $j$.

The EM-like meta-algorithm is defined in Algorithm \ref{alg:algorithm}. A specific implementation is discussed in the main text.

\begin{algorithm}[tb]
    \caption{EM-like meta-algorithm}
    \label{alg:algorithm}
    \textbf{Input}: r-MDP $\mathcal{M}_r$
    \begin{algorithmic}[1] 
        \STATE Arbitrarily initialize $(\alpha^i)_{i \in N}$ s.t. $\forall j: \exists i, \alpha^i(j) > 0$
        \WHILE{assignments or policies change}
            \STATE Perform M-step to update policies
            \STATE Perform E-step to update assignments
        \ENDWHILE
    \end{algorithmic}
\end{algorithm}

\begin{theorem}\label{theorem:app}
    Given an r-MDP $\mathcal{M}_r$, the EM-like meta-algorithm converges to a local maximum of $SW(\boldsymbol\alpha, \boldsymbol\pi).$ 
\end{theorem}

\begin{proof}
    Observe that an E-step may not decrease social welfare:

    \begin{equation*}
    \begin{split}
        & SW(\boldsymbol\alpha^*, \boldsymbol\pi) - SW(\boldsymbol\alpha, \boldsymbol\pi) =\\
        & \mathbb{E}_{\mathcal{T}_0} \sum_i \left[ \max_j V^{ij}(s_0) - \sum_j \alpha^i(j) V^{ij}(s_0) \right] \geq 0.
    \end{split}
    \end{equation*}

    Likewise, observe that an M-step may not decrease social welfare:

    \begin{equation*}
    \begin{split}
        & SW(\boldsymbol\alpha, \boldsymbol\pi^*) - SW(\boldsymbol\alpha, \boldsymbol\pi) =\\
        & \mathbb{E}_{\mathcal{T}_0} \sum_{i, j} \left[ \alpha^i(j) \left( V^i(s_0 \mid \pi^{j*}) - V^i(s_0 \mid \pi^j) \right) \right] \geq 0.
    \end{split}
    \end{equation*}

    Therefore, the iterative application of E-step and M-step monotonically increases social welfare until convergence. Because the number of possible assignments is finite, the convergence is guaranteed.
\end{proof}

\section{Hyperparameters and Technical Details}

\paragraph{PPO}

The PPO algorithm updates the policy to minimize the following loss function:

\begin{equation}\label{eq:loss_ppo_actor}
    \begin{split}
    L(\theta) = - \sum_{t \in B} \min[\rho_\theta(s_t, a_t) \tilde{A}(s_t, a_t),& \\  \min(\max(\rho_\theta(s_t, a_t), 1-\epsilon), 1+\epsilon) \tilde{A}(s_t, a_t)]&
    \end{split} 
\end{equation}

Our implementation of PPO is based on the PyTorch package \cite{paszke2019pytorch} for Python 3. We followed the procedure of \cite{shengyi2022the37implementation} to replicate the performance from the original paper \cite{schulman2017proximal}. Specifically, we implemented:

\begin{itemize}
    \item Orthogonal weight initialization;
    \item Generalized advantage estimation \cite{schulman2015high};
    \item Normalization of advantages over the batch (per policy);
    \item Entropy bonus to encourage exploration;
    \item Gradient norm clipping;
    \item Continuous actions via normal distributions plus reparameterization trick;
    \item State-independent log standard deviations as learnable parameters;
    \item Independent action components;
    \item Action clipping;
    \item Normalization and clipping of observations.
\end{itemize}

Actors and critics were both trained with ADAM optimizers \cite{kingma2014adam}. Furthermore, we used hyperparameters standard for MuJoCo environments:

\begin{itemize}
    \item Learning rate initialized at $0.0003$ and annealed throughout the training to $0.0001$;
    \item Entropy loss coefficient of $0.001$;
    \item GAE $\lambda = 0.95$;
    \item One hidden layer with 64 neurons;
    \item Batch of 2048 transitions, divided into mini-batches of 64 transitions;
    \item A batch is used for training for 10 epochs;
    \item PPO clipping parameter $\epsilon = 0.2$;
    \item ADAM $\epsilon = 10^{-5}$.
\end{itemize}

\paragraph{Our algorithms}

For our EM-like algorithm, we used a mixing coefficient $\lambda = 0.05$. For our end-to-end algorithm, we updated $\psi$ with ADAM optimizer with a learning rate of $0.002$ in the same backward passes as the policies $\theta_i$.

\section{Additional Histograms}

Figure \ref{fig:results_histograms_app} reports histograms of assignments learned by ours and baseline algorithms. These echo the conclusions in the main text: our algorithms divide the latent velocity space better than the baseline, and thus provide more personalization.

\begin{figure*}[t]
\begin{center}
    \begin{subfigure}{0.3\textwidth}
        \centering
        \caption{Ant, EM (ours)}
        \includegraphics[width=\linewidth]{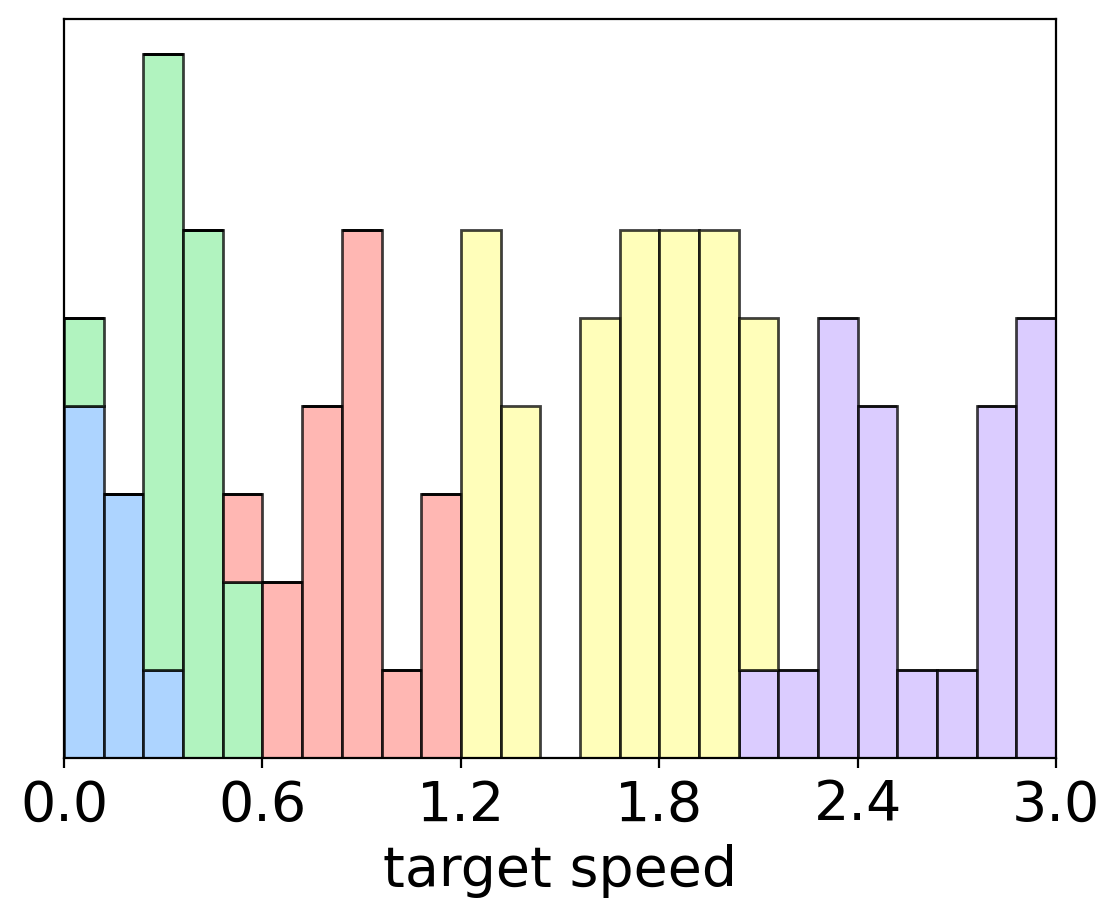}
    \end{subfigure}\hspace{10pt}%
    \begin{subfigure}{0.3\textwidth}
        \centering
        \caption{Ant, end-to-end (ours)}
        \includegraphics[width=\linewidth]{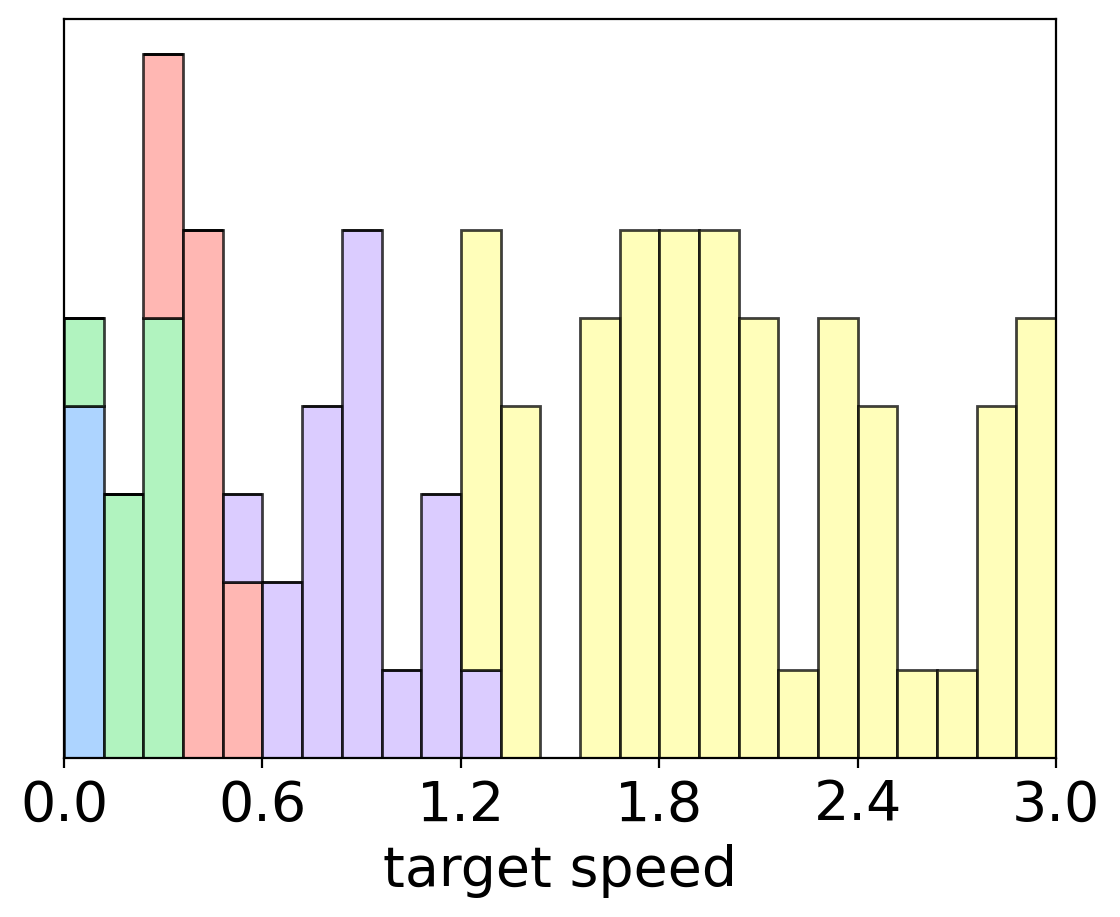}
    \end{subfigure}\hspace{10pt}%
    \begin{subfigure}{0.3\textwidth}
        \centering
        \caption{Ant, cluster (baseline)}
        \includegraphics[width=\linewidth]{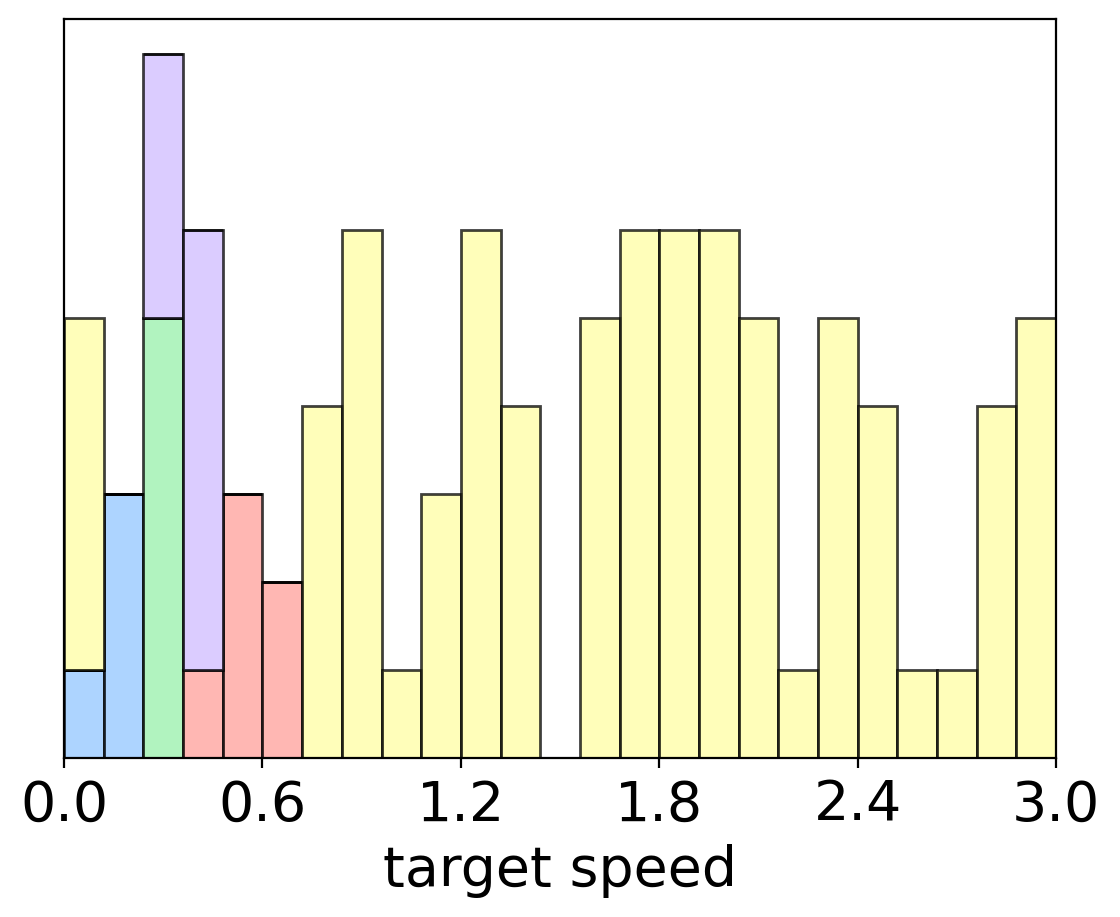}
    \end{subfigure}
    
    \begin{subfigure}{0.3\textwidth}
        \centering
        \caption{Hopper, EM (ours)}
        \includegraphics[width=\linewidth]{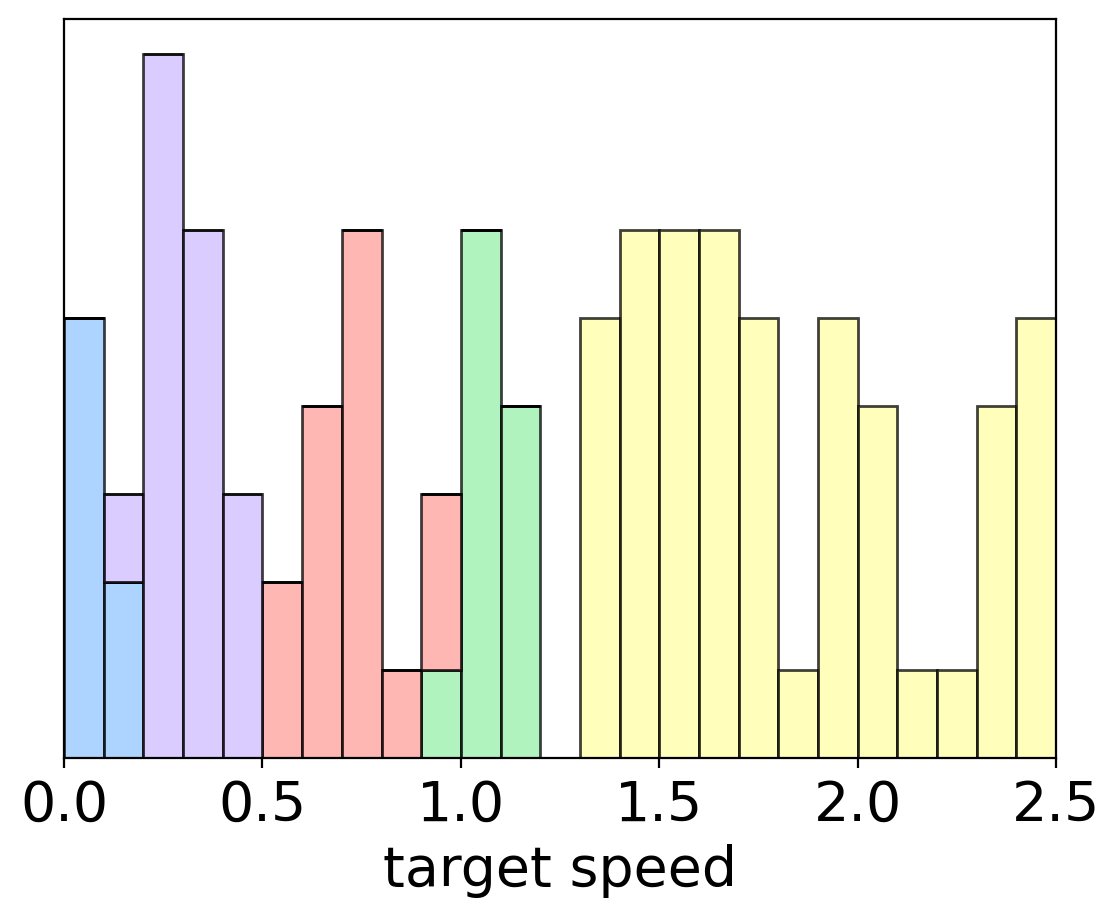}
    \end{subfigure}\hspace{10pt}%
    \begin{subfigure}{0.3\textwidth}
        \centering
        \caption{Hopper, end-to-end (ours)}
        \includegraphics[width=\linewidth]{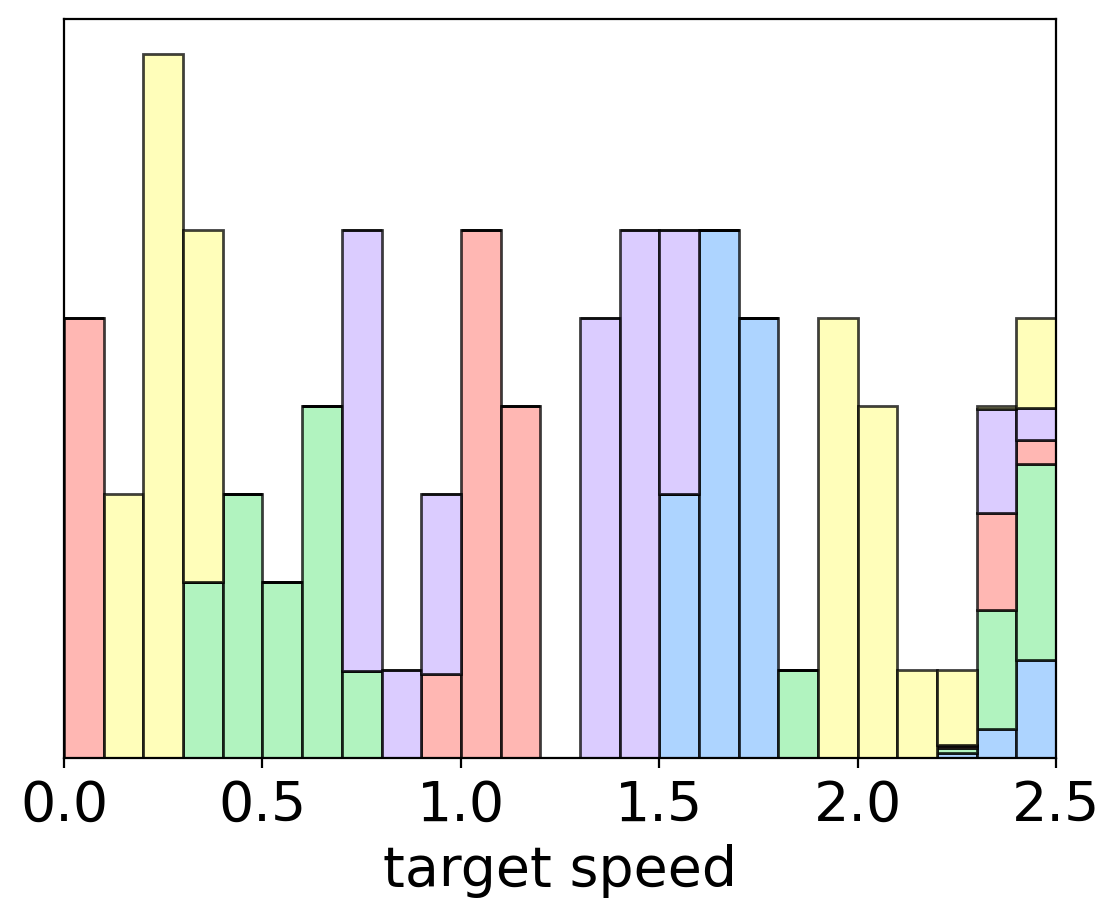}
    \end{subfigure}\hspace{10pt}%
    \begin{subfigure}{0.3\textwidth}
        \centering
        \caption{Hopper, cluster (baseline)}
        \includegraphics[width=\linewidth]{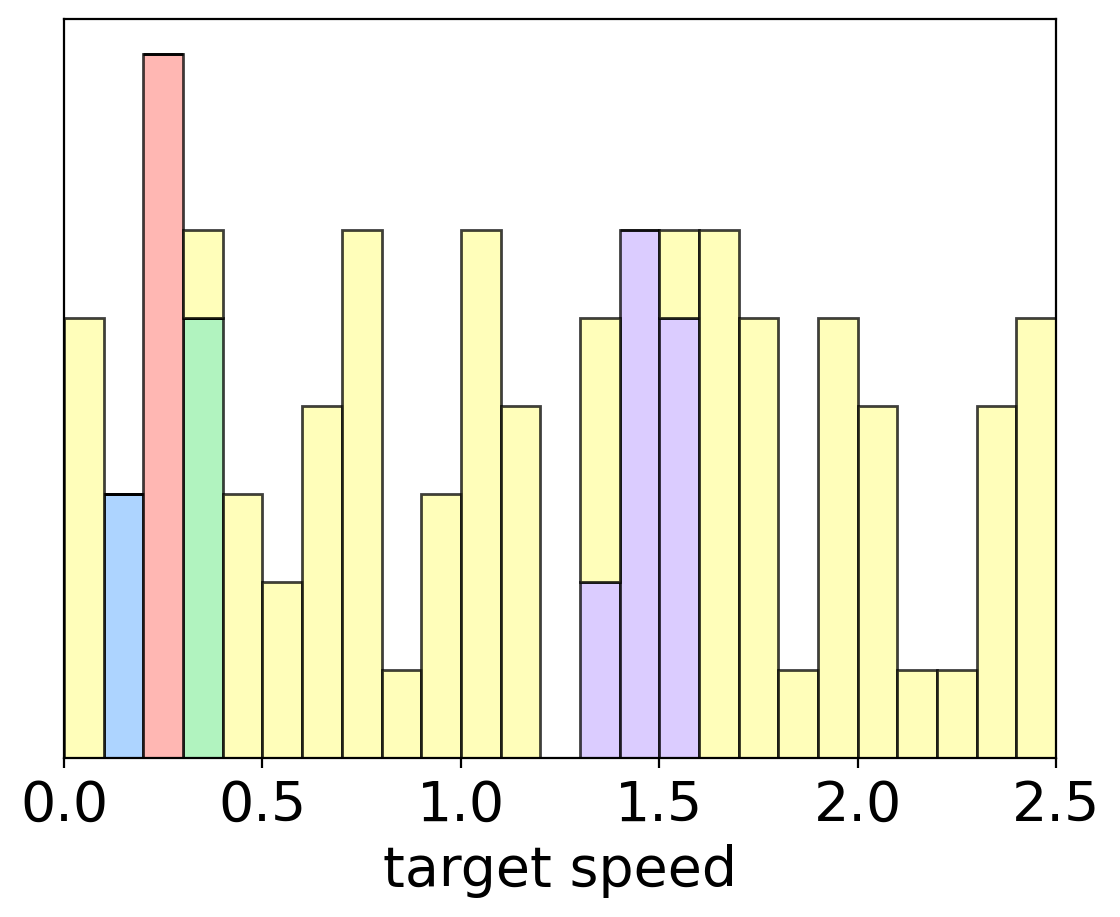}
    \end{subfigure}
    
    \begin{subfigure}{0.3\textwidth}
        \centering
        \caption{Walker2d, EM (ours)}
        \includegraphics[width=\linewidth]{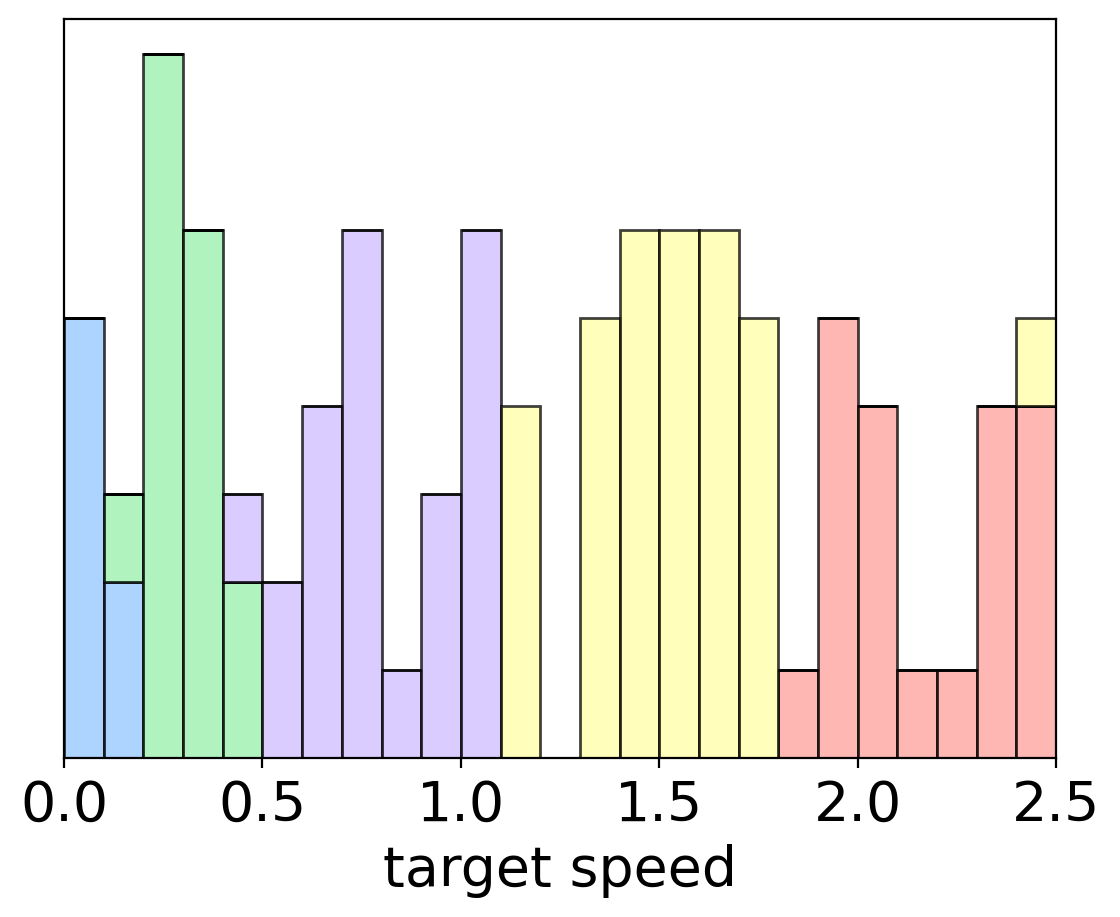}
    \end{subfigure}\hspace{10pt}%
    \begin{subfigure}{0.3\textwidth}
        \centering
        \caption{Walker2d, end-to-end (ours)}
        \includegraphics[width=\linewidth]{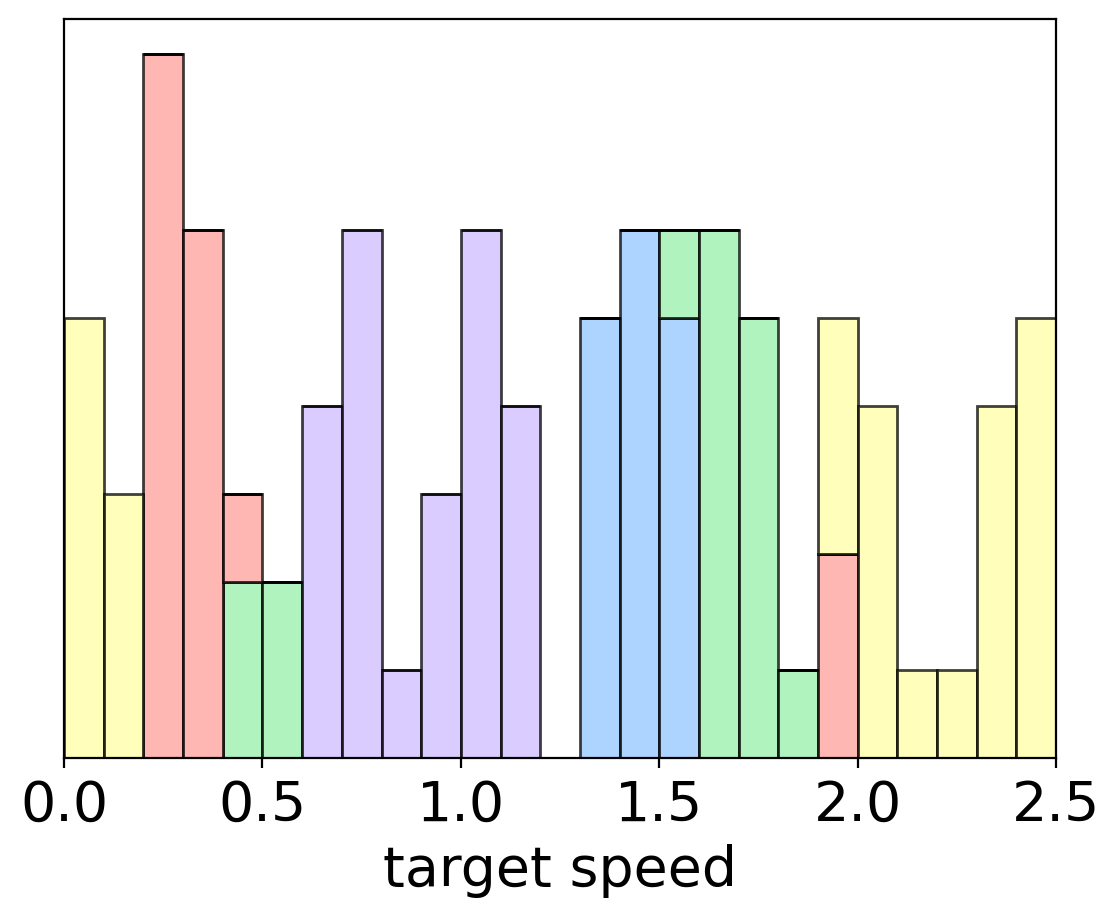}
    \end{subfigure}\hspace{10pt}%
    \begin{subfigure}{0.3\textwidth}
        \centering
        \caption{Walker2d, cluster (baseline)}
        \includegraphics[width=\linewidth]{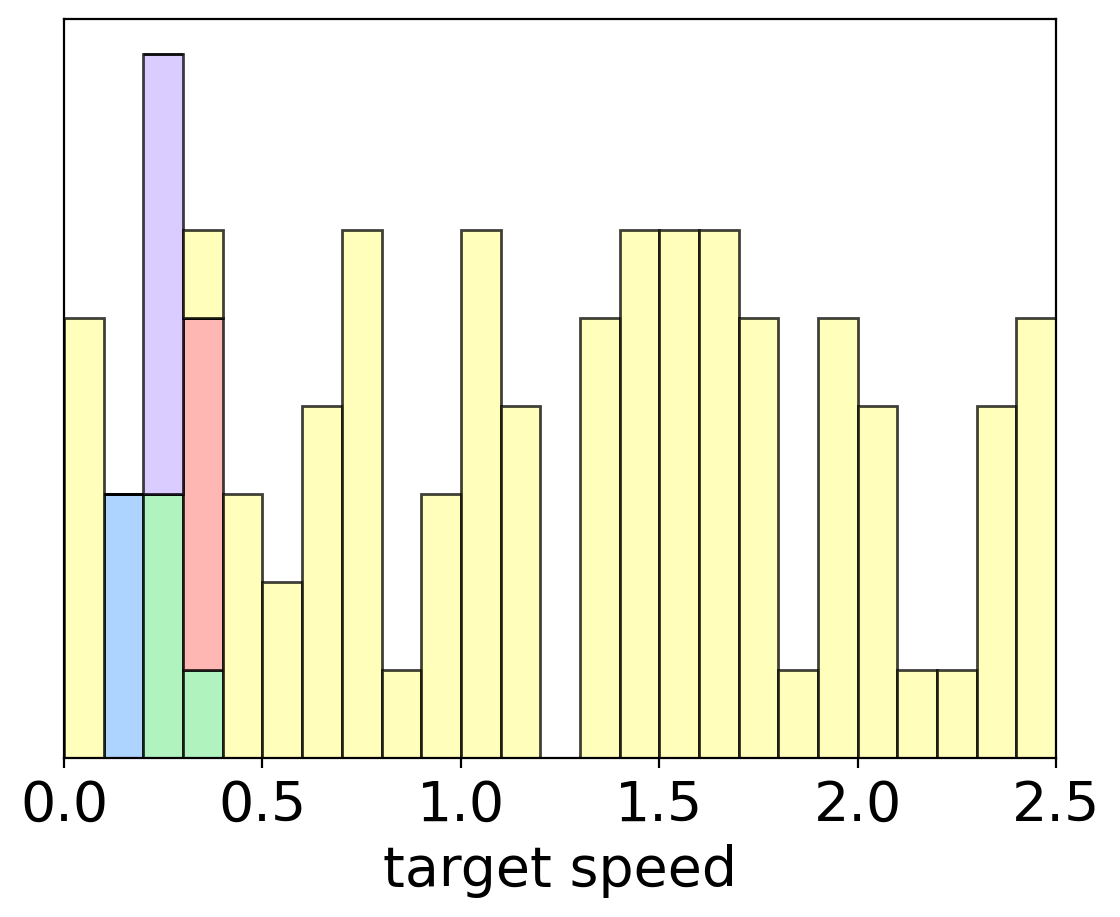}
    \end{subfigure}
\caption{Histograms of agent assignments learned by ours and baseline algorithms for $n=100$, $k=5$ in Ant, Hopper, and Walker2d ($0$-th random seed). Each color denotes one of five representatives and bars of this color denote the target velocities of agents assigned to this representative. The expected behavior is a division of the agents' velocities into five intervals of similar sizes, one for each representative.}
\label{fig:results_histograms_app}
\end{center}
\end{figure*}

